\documentclass[runningheads]{llncs} 

\usepackage[utf8]{inputenc} 
\usepackage[T1]{fontenc}    

\usepackage{hyperref}       
\usepackage{booktabs}       
\usepackage{nicefrac}       
\usepackage{microtype}      
\usepackage{tikz}

\usepackage{bm} 
\usepackage{import}
\usepackage{amsfonts,amsmath}

\newcommand{\br}{\boldsymbol{r}}

\newcommand{\Algo}{\mathcal{L}}            
\newcommand{\Reg}{\mathrm{Reg}}          
\newcommand{\BayReg}{\mathrm{BayReg}}  
\newcommand{\calO}{\mathcal{O}}               
\newcommand{\calE}{\mathcal{E}}               
\newcommand{\calS}{\mathcal{S}}

\newcommand{\calR}{\mathcal{R}}
\newcommand{\hatB}{\hat{\mathcal{B}}}

\newcommand{\calB}{\mathcal{B}}

\newcommand\bX{\boldsymbol{X}}
\newcommand\bx{\boldsymbol{x}}
\newcommand\bi{\boldsymbol{i}}

\newcommand\by{\boldsymbol{y}}

\newcommand\mSpace{\mathcal{X}}

\newcommand{\rhat}{\hat{r}}
\newcommand{\Qhat}{\hat{Q}}
\newcommand{\Phat}{\hat{P}}

\newcommand{\proba}[1]{\mathbb{P}\left(#1\right)}

\newcommand\Proba[1]{\mathbb{P}\left(#1\right)} 
\newcommand\sProba[1]{\mathbb{P}(#1)} 
\newcommand\esp[1]{\mathbb{E}\left[#1\right]} 
\newcommand\var[1]{\mathrm{var}\left[#1\right]} 
\newcommand\norm[1]{\left\|#1\right\|}      
\newcommand\abs[1]{\left|#1\right|}         
\newcommand\ceil[1]{\left\lceil#1\right\rceil}

\newcommand\p[1]{\left(#1\right)} 
\newcommand\ind[1]{\mathbb{I}_{\{#1\}}} 
\DeclareMathOperator*{\argmax}{arg\,max}

\newcommand\sto{^{\mathrm{stoc.pb}}} 
\newcommand\stoi{^{\mathrm{stoc.pb},j}} 

\newcommand\R{\mathbb{R}} 
\newcommand\M{\mathbb{M}}

\newcommand\ie{\emph{i.e.}}
\newcommand\eg{\emph{e.g.}}

\makeatletter
\newcommand{\customlabel}[2]{%
   \protected@write \@auxout {}{\string \newlabel {#1}{{#2}{\thepage}{#2}{#1}{}} }%
   \hypertarget{#1}{#2}
}
\makeatother

\usepackage{algorithm}
\usepackage{algorithmic}
\usepackage{graphicx}
\usepackage{subfigure}

\graphicspath{{}{figures/}}
\usepackage{times}

\title{Learning Algorithms for Markovian Bandits:\\Is Posterior Sampling more Scalable than Optimism?}
\titlerunning{MB-PSRL}

\author{
    Nicolas Gast \and
    Bruno Gaujal \and
    Kimang Khun
}

\institute{
    Univ. Grenoble Alpes, Inria, CNRS, Grenoble INP$^*$, LIG, 38000 Grenoble, France\\
    $^*$Institute of Engineering Univ. Grenoble Alpes\\
    \email{firstname.lastname@inria.fr}
}

\usepackage[textwidth=\linewidth]{todonotes}		

\begin{document}

\maketitle

\begin{abstract}%
   In this paper, we study the scalability of model-based algorithms learning the optimal policy of a discounted Markovian bandit  problem with $n$ arms.
 There are two categories of model-based reinforcement learning algorithms: Bayesian algorithms (like PSRL), and optimistic algorithms (like UCRL2 or UCBVI). 
While a naïve application of these  algorithms is not scalable because  the state-space is exponential in $n$, we construct variants  specially tailored to Markovian bandits (MB) that we call MB-PSRL, MB-UCRL2, and MB-UCBVI. They all have a low regret in $\tilde{O}(S\sqrt{nK})$ -- where $K$ is the number of episodes, $n$ is the number of arms and $S$ is the number of states of each arm. Up to a factor $\sqrt{S}$, these regrets match the lower bound of $\Omega(\sqrt{SnK})$ that we also derive. 

    Even if their theoretical regrets are comparable, the practical applicability of those algorithms varies greatly: We show that MB-UCRL2, as well as all  algorithms that use bonus on transition matrices, cannot be implemented efficiently (its {\it time} complexity is exponential in $n$). As MB-UCBVI does not use bonus on transition matrices, it can be implemented efficiently but our numerical experiments show that its empirical regret is large. Our Bayesian algorithm, MB-PSRL takes the best of both worlds: its running time is linear in the number of arms and its empirical regret is the smallest of all algorithms. This shows the power of Bayesian algorithms, that can often be easily tailored to the structure of the problems to learn.

\end{abstract}

\section{Introduction}
\label{sec:intro}

Markov Decision Processes (MDPs) are a powerful model to solve stochastic optimization problems. They suffer, however, from what is called the \emph{curse of dimensionality}: the state size of a Markov process is exponential in its number of dimensions, so that the complexity of computing an optimal policy is exponential in the number of dimensions of the problem. The same holds for general purpose reinforcement learning algorithm: they all have a regret and a runtime exponential in the number of dimensions, so they also suffer from the same curse. Very few MDPs are known to escape from this curse of dimensionality. One of the most famous examples is the Markovian bandit problem for which an optimal policy and its value can be computed in $O(n)$, where $n$ is the number of arms: The optimal policy can be computed by using the Gittins indices (computed locally) and its value can be computed by using retirement values (see for example \cite{whittle1996optimal}).

In this paper, we study a specialization of PSRL \cite{osbandMoreEfficientReinforcement2013a} to Markovian bandits, that we call Markovian bandit posterior sampling (MB-PSRL) that consists in using PSRL with a prior tailored to Markovian bandits. We show that the regret of MB-PSRL is sub-linear in the number of episodes and of arms. We also provide a regret guarantee for two optimistic algorithms that we call MB-UCRL2 and MB-UCBVI, and that are based respectively on UCRL2 \cite{auerNearoptimalRegretBounds2009} and UCBVI \cite{azar2017minimax}. They both use modified  confidence bounds adapted to Markovian bandit problems. The upper bound for their regret is similar to the bound for MB-PSRL. This shows that in terms of regret, the posterior sampling approach (MB-PSRL) and the optimistic approach (MB-UCRL2 and MB-UCBVI) scale well with the number of arms. We also provide a lower bound on the regret of any learning algorithm in Markovian bandit problems, which shows that the regret bounds that we obtain for all  algorithms are close to optimal.

The situation is radically different when considering the processing time: the runtime of MB-PSRL is linear in the number of arms, while the runtime of MB-UCRL2 is exponential in $n$. We show that this is not an artifact of our implementation of MB-UCRL2 by exhibiting a Markovian bandit problem for which being optimistic in each arm is not optimistic in the global MDP. This implies that UCRL2 and its variants \cite{bourel2020tightening,fruitEfficientBiasSpanConstrainedExplorationExploitation2018,talebiVarianceAwareRegretBounds2018,filippiOptimismReinforcementLearning2010a} cannot be adapted to have linear runtime in Markovian bandit problem unless an oracle is given. We argue that this non-scalability of UCRL2 and variants is not a property of optimistic approach but comes from the fact that UCRL2 relies on  extended value iteration \cite{auerNearoptimalRegretBounds2009} needed to deal with  upper confidence bounds on the transition matrices.  We show that MB-UCBVI, an optimistic algorithm that does not add bonus on transition probabilities and hence does not rely on extended value iteration,  does not suffer from the same problem. Its  regret is  sub-linear in the number of episodes, arms, and states per arm (although larger than the regret of both MB-PSRL and MB-UCRL2), and its  runtime is linear in the number of arms.

We conduct a series of numerical experiments to compare the performance of MB-PSRL,  MB-UCRL2 and MB-UCBVI.  They confirm the good behavior of MB-PSRL, both in terms of regret and computational complexity. These numerical experiments also show that the empirical regret of MB-UCBVI is larger than the regret of MB-PSRL and MB-UCRL2, confirming the comparisons between the upper bounds derived in Theorem \ref{thm:regret_upper_bound}. All this  makes MB-PSRL the better choice between the three learning algorithms.

\paragraph{Related Work}
We focus on Markovian bandit problem with discount factor $\beta<1$ and all
reward functions and transition matrices $(\br^a,Q^a)_{a\in\{1,\dots,n\}}$ are
unknown. A possible approach to learn such a problem is to ignore the problem structure and view the
Markovian bandit problem as a generic MDP. There are two main families of generic reinforcement learning algorithms with regret guarantees. The first one uses the \emph{optimism in face of uncertainty} (OFU) principle.  OFU methods build a confidence set
for the unknown MDP and compute an  optimal policy of the ``best'' MDP in the
confidence set, \emph{e.g.}, \cite{bourel2020tightening,fruitRegretMinimizationMDPs2017,azar2017minimax,auerNearoptimalRegretBounds2009,bartlettREGALRegularizationBased2009}. UCRL2 \cite{auerNearoptimalRegretBounds2009} is a well known OFU algorithm.  The second family uses a Bayesian approach, the posterior sampling method introduced in \cite{thompsonLikelihoodThatOne1933a}.  Such algorithms keep a
posterior distribution over possible MDPs and execute the optimal policy of a sampled MDP, see \emph{e.g.}, \cite{ouyangLearningUnknownMarkov2017b,agrawal17,gopalanThompsonSamplingLearning2015,osbandMoreEfficientReinforcement2013a}. PSRL is a classical example of Bayesian learning algorithm.  All these algorithms, based on OFU or on Bayesian principles, have sub-linear bounds on the regret, which means that they provably learn the optimal policy. Yet, applied as-is to Markovian bandit problems, these bounds grow exponentially with the number of arms.

Our work is not the first attempt to exploit the structure of a MDP to improve
learning.  Factored MDPs (the state space can be
factored into $n$ components) are investigated in
\cite{guestrinEfficientSolutionAlgorithms2003}, where asymptotic convergence to
the optimal policy is proved to scale polynomially in the number of components.
The regret of learning algorithms in factored MDP with a factored action space
is considered in 
\cite{tianMinimaxOptimalReinforcement2020,rosenbergOracleEfficientReinforcementLearning2020,xuReinforcementLearningFactored2020a,osbandNearoptimalReinforcementLearning2014a}. Our work differs substantially from these. 
First,  the Markovian bandit problem   is not a  factored MDP because  the action
space is global and  cannot be factored. 
Second, our reward is discounted over an infinite horizon while  factored MDPs have been analyzed with no discount.
Finally, and most importantly, the factored MDP framework assumes that the successive optimal policies are  computed by an unspecified solver. There is no guarantee that the time complexity of this solver scales linearly with the number of components, especially for OFU-based
algorithms.
For Markovian bandits, we get an  additional leverage: when all parameters are known, the Gittins index policy is known to be an optimal policy and its computational complexity is linear in the number of arms. This reveals an interesting difference between Bayesian and extented value based algorithms (the former being scalable and not the latter), which is not present in the literature about factored MDPs because such papers do not consider the time complexity.

Since index policies  scale with the number of arms, using Q-learning approaches
to learn such a policy is also popular \emph{e.g.},
\cite{avrachenkovWhittleIndexBased2020,fuQlearningWhittleIndex2019,duffQLearningBanditProblems1995a}. The authors of \cite{duffQLearningBanditProblems1995a} address the same Markovian bandit
problem as we do: their algorithm learns the optimal value in the
restart-in-state MDP \cite{katehakisMultiArmedBanditProblem1987a} for each
arm and uses Softmax exploration to solve the exploration-exploitation
dilemma. As mentioned on page $250$ in
\cite{auerFinitetimeAnalysisMultiarmed2002}, however, there exists no
finite-time regret bounds for this algorithm. Furthermore, tuning
its hyperparameters (learning rate and temperature) is rather delicate and
unstable  in practice. 

\section{Markovian Bandit Problem}
\label{sec:bandits}

In this section, we introduce the Markovian bandit problem and recall the notion of Gittins index when the parameters $(\br^a,Q^a)$ of all arms are known. 

\subsection{Definitions and Main Notations}
\label{ssec:def}

We consider a Markovian bandit problem with $n$ arms. Each arm
$\langle\calS^a,\br^a,Q^a\rangle$ for $a\in\{1,\dots, n\}=:[n]$ is a Markov
reward process with a finite state space $\calS^a$ of size $S$. Each arm has
a mean reward vector, $\br^a\in[0,1]^S$, and a transition matrix
$Q^a$.  When Arm~$a$ is activated in state $x_a\in\calS^a$, it moves to
state $y_a\in\calS^a$ with probability $Q^a(x_a,y_a)$.  This provides a reward
whose expected value is $r^a(x_a)$.  Without loss of generality, we assume
that the state spaces of the arms are pairwise distinct:
$\calS^a\cap\calS^b=\emptyset$ for $a\ne b$.  In the following, the state of an
arm $a$ will always be denoted with an index $a$: we will denote such a
state by $x_a$ or $y_a$.  As state spaces are disjoint, this allows us to
simplify the notation by dropping the index $a$ from the reward and transition
matrix: when convenient, we will denote them by $r(x_a)$ instead of $r^a(x_a)$
and by $Q(x_a,y_a)$ instead of $Q^a(x_a,y_a)$ since no confusion is possible. 

At time $1$, the global  state $\bX_1$ is distributed according to some initial
distribution $\rho$ over the global state space $\mSpace =
\calS^1{\times}\dots{\times}\calS^n$.  At time $t$, the decision maker observes
the states\footnote{Throughout the paper, we use capital letters (like $X_t$)
to denote random variables and small letter (like $\bx$) to denote their
realizations. Bold letters ($\bX_t$ or $\bx$) design vectors. Normal letters
($X_{t,a}$ or $x_a$) are for scalar values.} of all arms,
$\bX_t=(X_{t,1}\dots X_{t,n})$, and chooses which arm $A_t$ to activate.
This problem can be cast as a MDP -- that we denote by $M$ -- with state
space $\calE$ and action space \([n]\). Let $a\in[n]$ and $\bx,\by\in\calE$. If the
state at time $t$ is $\bX_t=\bx$, the chosen arm is $A_t=a$, then the agent
receives a random reward $R_t$ drawn from some distribution on \([0,1]\) with
mean \(r(x_a)\) and the MDP $M$ transitions to state $\bX_{t+1}=\by$ with
probability \(P^a(\bx,\by)\) that satisfies:
\begin{align}
    \label{eq:defP}
    P^a(\bx,\by) 
    &= 
    \left\{
        \begin{array}{@{}ll@{}}
            Q(x_a,y_a) & \text{ if $x_b=y_b$ for all $b\ne a$;}\\
            0 &\text{ otherwise}.
        \end{array}
    \right.
\end{align}
That is, the active arm makes a transition while the other arms remain in
the same state.

Let $\Pi$ be the set of deterministic policies, \emph{i.e.,} the set of
functions $\pi:\mSpace\mapsto[n]$.  For the MDP $M$, we denote by
$V_{M}^{\pi}(\bx)$ the expected cumulative discounted reward of ${M}$ under policy $\pi$  starting from an initial state $\bx$:
\[ V_{M}^\pi(\bx) {=}
\esp{\sum_{t=0}^{\infty} \beta^tR_t \mid \bX_0{=}\bx, A_t{=}\pi(\bX_t)}.\]
An alternative definition of $V$ is to consider a finite-horizon problem with a geometrically distributed length. Indeed, let $H$ be a time-horizon geometrically distributed with parameter $1-\beta>0$. We have
\begin{equation}
    \label{eq:V-vs_R}
    V_{M}^\pi(\bx)  {=}  \esp{\sum_{t=1}^{H} R_t \mid \bX_1{=}\bx, A_t{=}\pi(\bX_t)}.
\end{equation}

\begin{problem}
\label{pb:problem_1}
Given a Markovian bandit $M$ with $n$ arms, each is a Markov reward process $\langle\calS^a,\br^a,Q^a\rangle$ with a finite state space of size $S$, find a policy $\pi\colon\calS^1{\times}\dots{\times}\calS^n\mapsto[n]$ that maximizes $V_{M}^\pi(\bx)$ for any state $\bx$ distributed according to initial global state distribution $\rho$.
\end{problem}

A policy
$\pi_*$ is optimal for Problem~\ref{pb:problem_1} if $ V_{M}^{\pi_*}(\bx) \ge V_{M}^\pi(\bx) $ for all
$\pi\in\Pi$ and $\bx\in\calE$. By \cite{putermanMarkovDecisionProcesses1994},
such a policy exists and does not depend on $\bx$ (or $\rho$). It is given by Gittins index policy, defined below. 

\subsection{Gittins Index Policy}
\label{ssec:Gittins}

It is possible to compute an optimal
policy $\pi_*$ for Problem~\ref{pb:problem_1} in a reasonable amount of time using the so called Gittins indices:
Gittins defines in  \cite{gittinsBanditProcessesDynamic1979a}  the \emph{Gittins index} for any arm $a$ in state $x_a\in\calS_a$ as 
\begin{equation}
    \label{eq:eq_Gindex}
    \mathrm{GIndex}(x_a)=\sup_{\tau>0}\frac{\esp{\sum_{t=0}^{\tau-1}\beta^{t}r^a(Z_{t}) \mid Z_{0}=x_a}}{\esp{\sum_{t=0}^{\tau-1}\beta^{t}\mid Z_{0}=x_a}},
\end{equation}
where $Z$ is a Markov chain whose transitions are given by $Q^a$ and $\tau$ can
be any stopping time adapted to the natural filtration of $(Z_t)_{t\ge0}$. So, Gittins index can be considered as the maximal reward density over time of an arm at the given state.

It is shown in \cite{gittinsBanditProcessesDynamic1979a} that activating
the arm having the largest current index is an optimal policy. Such a policy
can be computed very efficiently: The computation of the indices of an arm
with $S$ states can be done in \(O(S^3)\) arithmetic operations, which means
that the computation of the Gittins index policy is linear in the number of arms
as it takes \(O(nS^3)\) arithmetic operations.
For more details about Gittins indices and optimality, we refer to
\cite{bookGittins,weberGittinsIndexMultiarmed1992}. For a survey on how to
compute Gittins indices, we refer to \cite{chakravortyMultiArmedBanditsGittins2014}, and to \cite{gast2022computing} for a recent paper that show how to compute Gittins index in subcubic time (\emph{i.e.}, $o(S^3))$ for each of the $n$ arms).

\section{Online Learning and Episodic Regret}
\label{sec:problem}

We now consider an extension of Problem~\ref{pb:problem_1} in which the decision maker does not know the transition matrices
nor the rewards. Our goal is to design a reinforcement learning algorithm that
learns the optimal policy from past observations.  Similarly to what is done
for finite-horizon reinforcement learning with deterministic horizon -- see
\eg, \cite{jinQLearningProvablyEfficient2018,azar2017minimax,osbandMoreEfficientReinforcement2013a} --
we consider a decision maker that faces a sequence of independent replicas of
the same Markovian bandit problem, where the transitions and the rewards 
are drawn independently for each episode.
What is new here is that  the time horizon $H$ is random and has a geometric distribution. It is drawn  independently for each episode. This implies that Gittins index policy is optimal for a decision maker that would know the transition matrices and rewards.

In this paper, we consider {\it episodic learning algorithms}. Let
$H_1,\ldots, H_k$ be the sequence of random episode lengths and let
$t_k:=1{+}\sum_{i=1}^{k-1}H_i$ be the starting time of the $k$th episode. Let 
\(\calO_{k-1}:=(\bX_1,A_1,R_1,\dots,\bX_{t_k-1},A_{t_k-1},R_{t_k-1})\) denote
the observations made prior and up to episode \(k\). An {\it Episodic Learning
Algorithm} \(\Algo\) is a function that maps observations \(\calO_{k-1}\) to
\(\Algo(\calO_{k-1})\), a probability distribution over all policies.  At the
beginning of episode $k$, the algorithm samples \(\pi_k\sim\Algo(\calO_{k-1})\)
and uses this policy during the whole $k$th episode. Note that one could also
design algorithms where learning takes place inside each episode. We will see
later that episodic learning as described here is enough to design algorithms
that are essentially optimal, in the sense given by Theorem~\ref{thm:regret_upper_bound} and Theorem~\ref{thm:lower_bound}.

For an  instance  \(M\) of a Markovian bandit problem and a total number of
episodes $K$, we denote by $\Reg(K,\Algo,M)$ the regret of a
learning algorithm \(\Algo\), defined as
\begin{align}
  \Reg(K,\Algo,M) :=  \sum_{k=1}^K 
      V^{\pi_*}_M(\bX_{t_k}) - V^{\pi_k}_M(\bX_{t_k})
      .\label{eq:regretTraj}
\end{align}
It is the sum over all episodes of the value of the optimal policy $\pi_*$ minus the value obtained by applying the policy $\pi_k$ chosen by the algorithm for episode $k$. In what follows, we will provide bounds on the expected regret.

A no-regret algorithm is an algorithm $\Algo$ such that its expected regret $\esp{\Reg(K,\Algo,M)}$ grows sub-linearly in the number of episodes $K$. This implies that the expected regret over episode $k$  converges to $0$ as $k$ goes to infinity. Such an algorithm learns an optimal policy of Problem~\ref{pb:problem_1}.

Note that, for discounted problems, an alternative regret definition (used for instance in \cite{zhou2021nearly}) is to use the non-episodic version $\sum_{t=1}^T (V^{\pi_*}_M(\bX_{t}) - V^{\pi_t}_M(\bX_{t}))$. In our definition \eqref{eq:regretTraj}, we use an episodic approach where the process is restarted according to $\rho$ after each episode of geometrically distributed length $H_k$.

\section{Learning Algorithms for Markovian Bandits}
\label{sec:learning_algorithms}

In what follows, we present three algorithms having a regret that grows like $\tilde{O}(S\sqrt{nK})$, that we call MB-PSRL, MB-UCRL2 and MB-UCBVI. As their names suggest, these algorithms are adaptation of PSRL, UCRL2 and UCBVI to Markovian bandit problems that intend to overcome the exponentiality in $n$ of their regret. The structure of the three MB-* algorithm is similar and is represented in Figure~\ref{fig:algos}. All algorithms are episodic learning algorithms. At the beginning of each episode, a MB-* learning algorithm computes a new policy $\pi_k$ that will be used during an episode of geometrically distributed length. The difference between the three algorithms lies in the way this new policy $\pi_k$ is computed. MB-PSRL uses posterior sampling while MB-UCRL2 and MB-UCBVI use optimism. We detail the three algorithms below. 

\begin{algorithm}[ht]
    \newcommand{\LineComment}[1]{\hfill\textit{\% #1}}
    \begin{algorithmic}[1]
        \INPUTS
            \STATE Discount factor $\beta$, initial distribution $\rho$ (and a prior distribution $(\phi^a)_{a\in[n]}$ for MB-PSRL)
        \ENDINPUTS
        \FOR{episodes $k=1,2,\dots$}
            \STATE \label{line:new_policy} Compute a new policy $\pi_k$ (using posterior sampling or optimism).
            \STATE Set $t_{k} \gets 1+\sum_{i=1}^{k-1}H_i$, sample $\bX_{t_k}\sim \rho$ and $H_k\sim \mathrm{Geom}(1-\beta)$.
            \FOR{$t\gets t_k$ {\bfseries to} $t_{k}+H_k-1$}
                \STATE Activate arm  $A_t=\pi_k(\bX_t)$.
                \STATE Observe $R_t$ and $\bX_{t+1}$.
            \ENDFOR
        \ENDFOR
    \end{algorithmic}
    \caption{Pseudo-code of the three MB-* algorithms.}
    \label{fig:algos}
\end{algorithm}

\subsection{MB-PSRL}
MB-PSRL starts with a prior distribution \(\phi^{a}\) over the parameters \((\br^a,Q^a)\). At the start of each episode $k$, MB-PSRL computes
a posterior distribution of parameters \(\phi^a(\cdot \mid\calO_{k-1})\) for
each arm $a\in[n]$ and samples parameters \((\br_k^a,Q_k^a)\) from
\(\phi^a(\cdot \mid\calO_{k-1})\) for each arm. Then, MB-PSRL uses \((\br_k^a,Q_k^a)_{a\in[n]}\)
to compute the Gittins index policy $\pi_k$ that is optimal for the sampled
problem.  The policy \(\pi_k\)  is then used for the whole episode \(k\). Note that as $\pi_k$ is a Gittins index policy, it can be computed efficiently. 

The difference between PSRL and MB-PSRL is mostly that MB-PSRL uses a prior distribution tailored to Markovian bandit.  The only hyperparameter of MB-PSRL is the prior distribution $\phi$. As we see in Appendix~\ref{apx:add_numerical},
MB-PSRL seems robust to the choice of the prior distribution, even if a
coherent prior gives a better performance than a misspecified prior, similarly
to what happens for Thompson's sampling \cite{russo2018tutorial}. 

\subsection{MB-UCRL2}
At the beginning of each episode $k$, MB-UCRL2 computes the following quantities for each state $x_a\in\calS^a$: $N_{k-1}(x_a)$ the number of times that Arm~$a$ is activated before episode $k$ while being in state $x_a$, and $\hat{r}_{k-1}(x_a)$, and $\hat{Q}_{k-1}(x_a,\cdot)$ are the empirical means of $r(x_a)$ and $Q(x_a,\cdot)$. We define the confidence bonuses
$b^r_{k-1}(x_a)=\sqrt{\frac{\log(2SnKt_{k})}{2\max\{1,N_{k-1}(x_a)\}}}$
and
$b^Q_{k-1}(x_a)=\sqrt{\frac{2\log(SnK2^St_k)}{\max\{1,N_{k-1}(x_a)\}}}$.
This defines a confidence set $\M_k$  as follows:  a Markovian bandit problem $M'$
is in $\M_k$ if for all $a\in[n]$ and $x_a\in\calS^a$: 
\begin{align}
    \label{eq:conf_rq}
    \abs{{r'}(x_a) - \hat{r}_{k-1}(x_a)} \le b^r_{k-1}(x_a) \text{ and }
    \lVert {Q'}(x_a,\cdot) - \hat{Q}_{k-1}(x_a,\cdot)\rVert_1 \le b^Q_{k-1}(x_a).            
\end{align}
MB-UCRL2 then chooses a policy $\pi_k$ that is optimal for the most optimistic problem $M_k\in\M_k$:
\begin{align}
    \pi_k \in \argmax_{\pi} \max_{M'\in\M_k}V^{\pi}_{M'}(\rho).
    \label{eq:EVI}
\end{align}
Note that as we explain later in Section~\ref{sec:OFU}, we believe that there is no  efficient algorithm to compute the best optimistic policy $\pi_k$ of Equation~\eqref{eq:EVI}.

Compared to a vanilla implementation of UCRL2, MB-UCRL2 uses the structure of the Markovian bandit problem: The constraints \eqref{eq:conf_rq} are on $Q$ whereas vanilla UCRL2 uses constraints on the full matrix $P$ (defined in \eqref{eq:defP}). This leads MB-UCRL2 to use the bonus term that scales as $\sqrt{S/N_{k-1}(x_a)}$ whereas vanilla UCRL2 would use the term in $\sqrt{S^n/N_{k-1}(\bx, a)}$.

\subsection{MB-UCBVI}
At the beginning of episode $k$, MB-UCBVI uses the same quantities $N_{k-1}(x_a)$, $\hat{r}_{k-1}(x_a)$, and $\hat{Q}_{k-1}(x_a,\cdot)$ as MB-UCRL2. The difference lies in the definition of the bonus terms. While MB-UCRL2 uses a bonus on the reward and on the transition matrices, MB-UCBVI defines a bonus $b_{k-1}(x_a){:=}\frac{1}{1-\beta}\sqrt{\frac{\log(2SnKt_k)}{2\max\{1,N_{k-1}(x_a)\}}}$ that is used on the reward only. MB-UCBVI computes the Gittins index policy $\pi_k$ that is optimal for the bandit problem $(\hat{\br}_{k-1}^a {+}b_{k-1}^a, \Qhat_{k-1}^a)_{a\in[n]}$. 

Similarly to the case of UCRL2, a vanilla implementation of UCBVI would use a bonus that scales exponentially with the number of arms. MB-UCBVI makes an even  better use of the structure of the learned problem because the optimistic MDP $(\hat{\br}_{k-1}^a {+}b_{k-1}^a, \Qhat_{k-1}^a)_{a\in[n]}$ is still a Markovian bandit problem. This implies that the optimistic policy $\pi_k$ is a Gittins index policy, and that can therefore be computed efficiently.

\section{Regret Analysis}
\label{sec:analysis}

In this section, we first present upper bounds on the expected regret of the three learning algorithms.  These bounds are sub-linear in the number of episodes (hence the three algorithms are no-regret algorithms) and sub-linear in the number of arms. We then derive a minimax lower bound on the regret of any learning algorithm in the Markovian bandit problem. 

\subsection{Upper Bounds on Regret}
\label{ssec:upper_bound_psrl}

The theorem below provides upper bounds on the expected regret of the three algorithms presented in Section~\ref{sec:learning_algorithms}. Note that since MB-PSRL is a Bayesian algorithm, we consider its \emph{bayesian regret}, that is the expectation over all possible model. More precisely, if the unknown MDP $M$ is drawn from a prior distribution $\phi$, the \textit{Bayesian regret} of a learning algorithm $\Algo$ is $\BayReg(K,\Algo,\phi) = \mathbb{E}[\Reg(K,\Algo,M)]$, where the expectation is taken over all possible values of $M \sim \phi$ and all possible runs of the algorithm. The expected regret $\esp{\Reg(K,\Algo,M)}$ is defined by taking the expectation over all possible runs of the algorithm.

\begin{theorem}
    \label{thm:regret_upper_bound}
    Let $f(S,n,K,\beta)=Sn\p{\log K/(1{-}\beta)}^{2}+ \sqrt{SnK}\p{\log K/(1{-}\beta)}^{3/2}$. There exists universal constants $C, C'$ and $C''$ independent of the model (i.e., that do not depend on $S$, $n$, $K$ and $\beta$) such that:
    \begin{itemize}
        \item For any prior distribution $\phi$:
        \begin{align*}
            \BayReg(K,\text{MB-PSRL},\phi) &\le C\p{\sqrt{S} {+}\log\frac{SnK\log K}{1-\beta}} f(S,n,K,\beta),
        \end{align*}
        \item For any Markovian bandit model $M$:
    \begin{align*}
        \esp{\Reg(K,\text{MB-UCRL2},M)} &\le C'\p{\sqrt{S} {+}\log\frac{SnK\log K}{1-\beta}} f(S,n,K,\beta),\\
        \esp{\Reg(K,\text{MB-UCBVI},M)} &\le C''\p{\frac{\sqrt{S}}{1-\beta}}\p{\log\frac{SnK\log K}{1-\beta}} f(S,n,K,\beta),
    \end{align*}
    \end{itemize}
\end{theorem}

We provide a sketch of proof below. The detailed proof is provided in Appendix~\ref{apx:proof_thm1} in the supplementary material.

This theorem calls for several comments. First, it shows that when $K\ge Sn/(1-\beta)$, the regret of MB-PSRL and MB-UCRL2 is smaller than 
\begin{equation}
    \label{eq:regret_tilde}
    \tilde{O}\left(\frac{S\sqrt{nK}}{(1-\beta)^{3/2}}\right),
\end{equation}
where the notation $\tilde{O}$ means that all logarithmic terms are removed. The regret of MB-UCBVI has an extra $1/(1-\beta)$ factor. 

Hence, the regret of the three algorithms is sub-linear in the number of episodes $K$ which means that they all  are no-regret algorithms. This regret bound is sub-linear in the number of arms which is very significant in practice when facing a large number of arms. Note that directly applying PSRL, UCRL2 or UCBVI would lead to a regret in   $\tilde{O}\left(S^n\sqrt{nK}\right)$ or $\tilde{O}\left(\sqrt{nS^nK}\right)$,  which is exponential in $n$.

Second, the upper bound on the expected regret of MB-UCRL2 (and of MB-UCBVI) is a guarantee for a specific problem $M$ while the bound on Bayesian regret of MB-PSRL is a guarantee in average overall the problems drawn from the prior $\phi$. Hence, the bounds of MB-UCRL2 and MB-UCBVI are stronger guarantee compared to the one of MB-PSRL. Yet, as we will see later in the numerical experiments reported in Section \ref{sec:numerical}, MB-PSRL seems to have a smaller regret in practice, even when the  problem does not follow the correct prior. 

Finally, our bound \eqref{eq:regret_tilde} is linear in $S$, the state size of each arm.
 Having a regret bound linear in the state space size is currently state-of-the-art for Bayesian algorithms, \eg, \cite{agrawal17,ouyangLearningUnknownMarkov2017b}. For optimistic algorithms, the best regret bounds are linear in the square root of the state-space size because they use  Bernstein's concentration bounds instead of Weissman's inequality \cite{azar2017minimax}, yet this approach does not work in the discounted case because of the random length of episodes. UCBVI has also been studied in the discounted case in \cite{zhou2021nearly}. However they use  with a different definition of the regret,  making their bound on the regret hard to compare with ours.


\subsubsection*{Sketch of proof}
    A crucial ingredient of our proof is to work with value function over a random finite time horizon ($W$ defined below), instead of working directly with the discounted value function $V$. For a given model $M$, and a stationary policy $\pi$, a horizon $H$ and a time $h\le H$, we define by $W_{M,{h:H}}^\pi(\bx)$ the value function of a policy $\pi$ over the finite time horizon ${H-h+1}$ when starting in $\bx$ at time $h$. It is defined as
    \begin{equation}
        \label{eq:discounted_reward}
        W_{M,{h:H}}^\pi(\bx)  =  r^\pi(\bx) {+}\sum_{\by\in\calE}P^\pi(\bx,\by)W_{M,{h+1:H}}^\pi(\by),
    \end{equation}
    with $W_{M,{H:H}}^\pi(\bx)=r^\pi(\bx)$ and where $r^\pi$ and $P^\pi$ are reward vector and state transition matrix when following policy $\pi$.

    By definitions of $W$ in \eqref{eq:discounted_reward} and $V$ in \eqref{eq:V-vs_R}, for a fixed model $M$, a policy $\pi$ and a state $\bx$, and a time-horizon $H$ that is geometrically distributed, one has $V^{\pi}_M(\bx)=\esp{W_{M,1:H}^\pi(\bx)}$.

    This characterization is important in our proof. Since the episode length $H_k$ is independent of the observations available before episode $k$, $\calO_{k-1}$, for any policy $\pi_k$ that is independent of $H_k$, one has
    \begin{align}
        \esp{V^{\pi_k}_M(\bX_{t_k}) \mid {\calO_{k-1}}, \pi_k } 
        &= \esp{W^{\pi_k}_{M,1:H_k}(\bX_{t_k}) \mid \calO_{k-1}, \pi_k}.
        \label{eq:equivalence_V_W}
    \end{align}
    In the above Equation~\eqref{eq:equivalence_V_W}, the expectation is taken over all initial state $\bX_{t_k}$ and all possible horizon $H_k$.
    
    Equation~\eqref{eq:equivalence_V_W} will be very useful in our analysis as it allows us to work with either $V$ or $W$ interchangeably. While the proof of MB-PSRL could be done by only studying the function $W$, the proof of MB-UCRL2 and MB-UCBVI will use the expression of the regret as a function of $V$ to deal with the non-determinism. Indeed, at episode $k$, all algorithms compare the optimal policy $\pi_*$ (that is optimal for the true MDP $M$) and a policy $\pi_k$ chosen by the algorithm (that is optimal for a MDP $M_k$ that is either sampled by MB-PSRL or chosen by an optimistic principle).  The quantity $\Delta_k := W_{M,1:H_k}^{\pi_*}(\bX_{t_k}) - W_{M_k,1:H_k}^{\pi_k}(\bX_{t_k})$ equals:
    \begin{align}
        \underbrace{W_{M,1:H_k}^{\pi_*}(\bX_{t_k}) {-} W_{M_k,1:H_k}^{\pi_k}(\bX_{t_k})}_{\customlabel{eq:A}{(A)}}%
        + \underbrace{W_{M_k,1:H_k}^{\pi_k}(\bX_{t_k}) {-} W_{M,1:H_k}^{\pi_k}(\bX_{t_k})}_{\customlabel{eq:B}{(B)}}.
        \label{eq:proofAB}
    \end{align}
    The analysis of the term \eqref{eq:B} is similar for the three algorithms: it is bounded by the distance between the sampled MDP $M_k$ and the true MDP $M$ that can in turn be bounded by using a concentration argument (Lemma~\ref{lem:concentration}) based on Hoeffding's and Weissman's inequalities.  Compared with the litterature \cite{azar2017minimax,ouyangLearningUnknownMarkov2017b}, our proof leverages on taking conditional expectations, making all terms whose conditional expectation is zero disappear. One the of main technical hurdle is to deal with the $K$ random episodes $H_1,\ldots, H_k$.   This is also new in our approach compared to the classical analysis of finite horizons regrets. 

    The analysis of \eqref{eq:A} depends heavily on the algorithm used. The easiest case is PSRL: As our setting is Bayesian, the expectation of the first term \eqref{eq:A} with respect to the model is zero (see Lemma~\ref{lem:bayesian}).   The case of MB-UCRL2 and MB-UCBVI are harder. In fact, our bonus terms are specially designed so that $V_{M_k}^{\pi_k}(\bx)$ is an optimistic upper bound of the true value function with high probability, that is:
    \begin{align}
        \label{eq:optimism}
        V_{M_k}^{\pi_k}(\bx) = \max_{\pi} \max_{M'\in\mathbb{M}_k} V_{M'}^{\pi}(\bx) \ge V_M^{\pi_*}(\bx).
    \end{align}
    This requires the use of  $V$ and not $W$ and is used to show that the expectation of the term \eqref{eq:A} of \eqref{eq:proofAB} cannot be too positive. 

\subsection{Minimax Lower Bound}
\label{ssec:lowerbound}

After obtaining upper bounds on the regret, a natural question is: can we do better? Or in other terms, does there exist a learning algorithm with a smaller regret? To answer this question, the metric used in the literature is the notion of minimax lower bound: for a given set of parameters $(S,n,K,\beta)$, a minimax lower bound is a lower bound on the quantity 
$\inf_{\Algo}\sup_{M}\Reg(K,\Algo,M)$,  where the supremum is taken among all possible models that have parameters $(S,n,K,\beta)$ and the infimum is taken over all possible learning algorithms. The next theorem provides a lower bound on the Bayesian regret. It is therefore stronger than a minimax bound for two reasons: First, the Bayesian regret is an average over models, which means that there exists at least one model that has a larger regret than the Bayesian lower bound; And second,  in Theorem~\ref{thm:lower_bound}, we allow the algorithm to depend on the prior distribution $\phi$ and to use this information.
\begin{theorem}[Lower bound]
    \label{thm:lower_bound}
    For any state size $S$, number of arms $n$, discount factor $\beta$ and number of episodes $K\ge 16S$, there exists a prior distribution $\phi$ on Markovian bandit problems with parameters $(S,n,K,\beta)$ such that, for any learning algorithm $\Algo$:
    \begin{align}
        \label{eq:thm2}
        \BayReg(K,\Algo,\phi) \geq \frac1{60}\sqrt{\frac{{SnK}}{(1-\beta)}}.
    \end{align}
\end{theorem}
The proof is given in Appendix~\ref{apx:sketch_of_proof_lower} and uses a counterexample inspired by the one of \cite{auerNearoptimalRegretBounds2009}. Note that for general MDPs, the minimax lower bound  obtained in \cite{osband2016lower,auerNearoptimalRegretBounds2009} says that  a learning algorithm cannot have a regret smaller than $\Omega\big(\sqrt{\tilde{S}\tilde{A}\tilde{T}}\big)$, where $\tilde{S}$ is the number of states of the MDP, $\tilde{A}$ is the number of actions and $\tilde{T}$ is the number of time steps. 
Yet, the lower bound of \cite{osband2016lower,auerNearoptimalRegretBounds2009} is not  directly applicable to our case with $\tilde{S}=S^n$ because Markovian bandit problems are very specific instances of MDPs and this can be exploited by the  learning algorithm.
Also note that this lower bound on the Bayesian regret is also a lower bound on the expected regret of any non-Bayesian algorithm for any MDP model $M$.

Apart from the logarithmic terms, the lower bound provided by Theorem~\ref{thm:lower_bound} differs from the bound of Theorem~\ref{thm:regret_upper_bound} by a factor $\sqrt{S}/(1-\beta)$. This factor is similar to the one observed for PSRL and UCRL2~\cite{osbandMoreEfficientReinforcement2013a,auerNearoptimalRegretBounds2009}. There are various factors that could explain this. We believe that the extra factor  $1/(1-\beta)$ might be half due to the episodic nature of MB-PSRL and MB-UCRL2 (when $1/(1-\beta)$ is large, algorithms with internal episodic updates might have smaller regret) and half due to the fact that the lower bound of Theorem~\ref{thm:lower_bound} is not optimal and could include a term $1/\sqrt{1-\beta}$ (similar to the term $O(\sqrt{D})$ of the lower bound of \cite{osband2016lower,auerNearoptimalRegretBounds2009}). The factor $\sqrt{S}$ between our two bounds comes from our use of Weissman's inequality. It might be possible that our regret bounds  are not optimal with respect to this term although such an improvement cannot be obtained using the same approach as in  \cite{azar2017minimax}. 

\section{Scalability of Learning Algorithms For Markovian Bandits}

Historically, Problem~{\ref{pb:problem_1}} was considered unresolved until \cite{gittinsBanditProcessesDynamic1979a} proposed Gittins indices. This is because previous solutions were based on Dynamic Programming in the global MDP which are computationally expensive. Hence, after establishing regret guarantees, we are now interested in the computational complexity of our learning algorithms, which is often disregarded in the learning litterature.

\subsection{MB-PSRL and MB-UCBVI are scalable}
\label{sec:OFU}



If one excludes the simulation of the MDP, the computational cost of MB-PSRL and MB-UCBVI of each episode is low. For MB-PSRL, its cost is essentially due to three components: Updating the observations, sampling from the posterior distribution and computing the optimal policy. The first two are relatively fast when the conjugate posterior has a closed form: updating the observation takes $O(1)$ at each time, and sampling from the posterior can be done in $O(nS^2)$ -- more details on posterior distributions are given in Appendix~\ref{apx:algos}. When the conjugate posterior is implicit (\emph{i.e.}, under the integral form), the computation can be higher but remains linear in the number of arms. For MB-UCBVI, the cost is due to two components: computing the bonus terms and computing the Gittins policy for the optimistic MDP. Computing the bonus is linear in the number of bandits and the length of the episode. As explained in Section~\ref{ssec:Gittins}, the computation of the Gittins index policy for a given problem can be done in $O(nS^3)$. Hence, MB-PSRL and MB-UCBVI successfully escape from the curse of dimensionality.

\subsection{MB-UCRL2 is not scalable because it cannot use an Index Policy}
\label{ssec:no-OFU}

While MB-UCRL2 has a regret equivalent to the one of MB-PSRL, its computational complexity, and in particular the complexity of computing an \emph{optimistic} policy that maximizes \eqref{eq:EVI} does not scale with $n$. Such a policy can be computed by using \emph{extended value iteration} \cite{auerNearoptimalRegretBounds2009}. This computation is polynomial in the number of states of the global MDP and is therefore exponential in the number of arms, precisely $O(nS^{2n})$. 
For MB-PSRL (or MB-UCBVI), the computation is easier because the sampled (optimistic) MDP is a Markovian bandit problem. Hence, using Gittins Theorem, computing the optimal policy can be done by computing local indices. In the following, we show that it is not possible to solve \eqref{eq:EVI} by using local indices. This suggests that MB-UCRL2 (nor any of the modifications of UCRL2's variants that would use extended value iteration) cannot be implemented efficiently.

\medskip

More precisely, to find an optimistic policy (that satisfies \eqref{eq:optimism}), UCRL2 and its variants, \eg, KL-UCRL \cite{filippiOptimismReinforcementLearning2010a}, compute a policy $\pi_k$ that is optimal for the most optimistic MDP in $\M_k$. This can be done by using extended value iteration. We now show that this cannot be replaced by the computation of local indices.

Let us consider that the estimates and confidence bounds for a given arm $a$ are $\hatB^a=(\hat{\br}^a,\hat{Q}^a,b^r_a,b^{Q}_a)$. We say that an algorithm computes indices locally for  Arm $a$ if for each $x_a\in\calS^a$, it computes an index $I^{\hatB^a}(x_a)$ by using only $\hatB^a$ but not $\hatB^b$ for any $b\ne a$. We denote by $\pi^{I(\hatB)}$ the index policy that uses index $I^{\hatB^a}$ for arm $a$ and by $\M(\hatB)$ the set of Markovian bandit problems $M'$ that satisfy \eqref{eq:conf_rq}.
\begin{theorem}
    \label{thm:no_OFU}
    For any algorithm that computes indices locally, there exists a Markovian bandit problem $M$, an initial state $\bx$ and estimates $\hatB^a=(\hat{\br}^a,\hat{Q}^a,b^r_a, b^Q_a)$ such that $M\in\mathbb{M}(\hatB)$ and
    \begin{align*}
        \sup_{M'\in\mathbb{M}(\hatB)}V^{\pi^{I(\hatB)}}_{M'}(\bx) < \sup_\pi V^{\pi}_{M}(\bx).
    \end{align*}
\end{theorem}
\begin{proof}
    The proof presented in Appendix~\ref{apx:proof_OFU} is obtained by constructing a set $\M$ and two MDPs $M_1$ and $M_2$ in $\M$ such that \eqref{eq:optimism} cannot hold simultaneously for both $M_1$ and $M_2$. 
\end{proof}
This theorem implies that one cannot define local indices such that \eqref{eq:optimism} holds for all bandit problems $M\in\M_k$. Yet, the use of this inequality is central in the regret analysis of UCRL2 (see the proof of UCRL2 in \cite{auerNearoptimalRegretBounds2009}). This implies that the current methodology to obtain regret bounds for UCRL2 and its variants, \eg,\cite{bourel2020tightening,fruitEfficientBiasSpanConstrainedExplorationExploitation2018,talebiVarianceAwareRegretBounds2018,filippiOptimismReinforcementLearning2010a}, that use Extended Value Iteration is not applicable to bound the regret of their modified version that computes indices locally. 

Note that for any set $\M$ such that $M \in \M$, there still exists an index policy $\pi^{\mathrm{ind}}$ that is optimistic because all MDPs in $\M$ are Markovian bandit problems. This optimistic index policy satisfies
\begin{align*}
  \sup_{M'\in\M}V^{\pi^{\mathrm{ind}}}_{M'} \ge \sup_\pi V^{\pi}_{M}.
\end{align*}
This means that restricting to index policies is not a restriction for optimism. What Theorem~\ref{thm:no_OFU} shows is that an optimistic index policy can be defined only after the most optimistic MDP $M\in\M$ is computed and computing optimistic policy and $M$ simultaneously depends on the confidence sets of all arms.

Therefore, we  believe that UCRL2 and its variants cannot compute optimistic  policy locally: they should all require the joint knowledge of all $(\hatB^a)_{a\in[n]}$.
\section{Numerical Experiments}
\label{sec:numerical}

In complement to our theoretical analysis, we report, in this section, the performance of our three algorithms in a model taken from the literature. 
The model is an environment with 3 arms, all following a Markov chain that is obtained by applying the optimal policy on the river swim MDP. A detailed description is given in Appendix~\ref{apx:algos}, along with all hyperparameters that we used. Our numerical experiments suggest that MB-PSRL outperforms other algorithms in term of average regret and is computationally less expensive than other algorithms. To ensure reproducibility, the code and data of our experiments are available (link to  GitHub repository hidden for double blind review).

\paragraph{Performance Result}
We investigate the average regret and policy computation time of each algorithm.
To do so, we run each algorithm for $80$ simulations and for $K=3000$ episodes per simulation. We arbitrarily choose the discount factor $\beta=0.99$. In \figurename~\ref{fig:randomwalk_3b4s}, we show the average cumulative regret of the 3 algorithms. We observe that the average regret of MB-UCBVI is larger than those of MB-PSRL and MB-UCRL2.
Moreover, we observe that MB-PSRL obtains the best performance and that its regret seems to grow slower than $O(\sqrt{K})$. This is in accordance to what was observed for PSRL in \cite{osbandMoreEfficientReinforcement2013a}. 
Note that the expected number of time steps after $K$ episodes is $K/(1-\beta)$ which means that in our setting with $K=3000$ episodes there are $300\,000$ time steps in average. 
In \figurename~\ref{fig:randomwalk_cpt_3b4s}, we compare the computation time of the various algorithms. We observe that the computation time (the $y$-axis is in log-scale) of MB-PSRL and MB-UCBVI, the index-based algorithms, are the fastest by far. 
Moreover, the computation time of these algorithms seem to be independent of the number of episodes. 
These two figures show that MB-PSRL has the smallest regret and computation time among all compared algorithms.

\begin{figure}
    \centering
    \subfigure[Average cumulative regret in function of the number of episodes.]{
        \label{fig:randomwalk_3b4s}
        \includegraphics[width=0.47\linewidth]{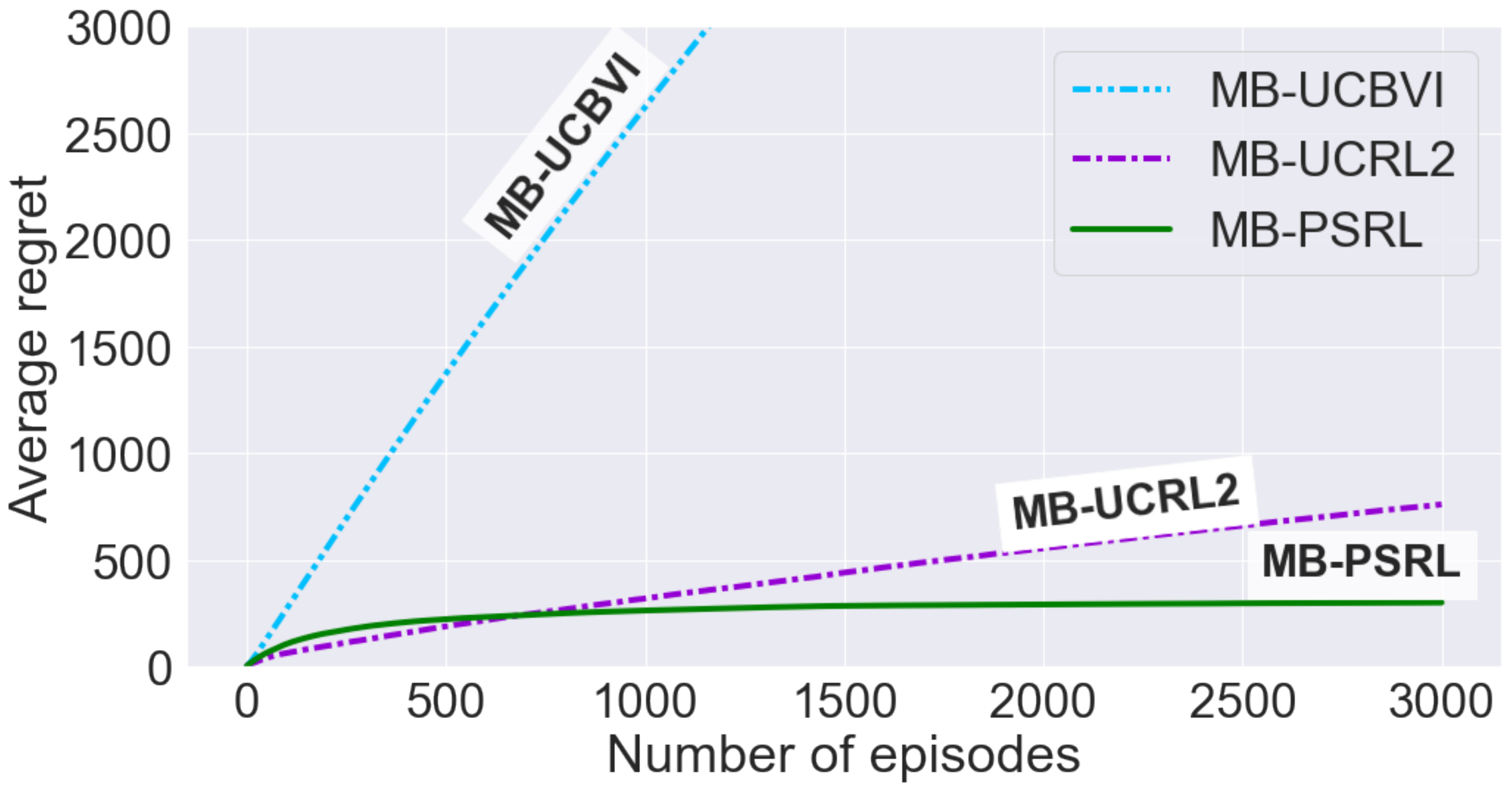} 
    }\hfill
    \subfigure[Average runtime per episode. The vertical axis is in log-scale.]{
        \label{fig:randomwalk_cpt_3b4s}
        \includegraphics[width=0.47\linewidth]{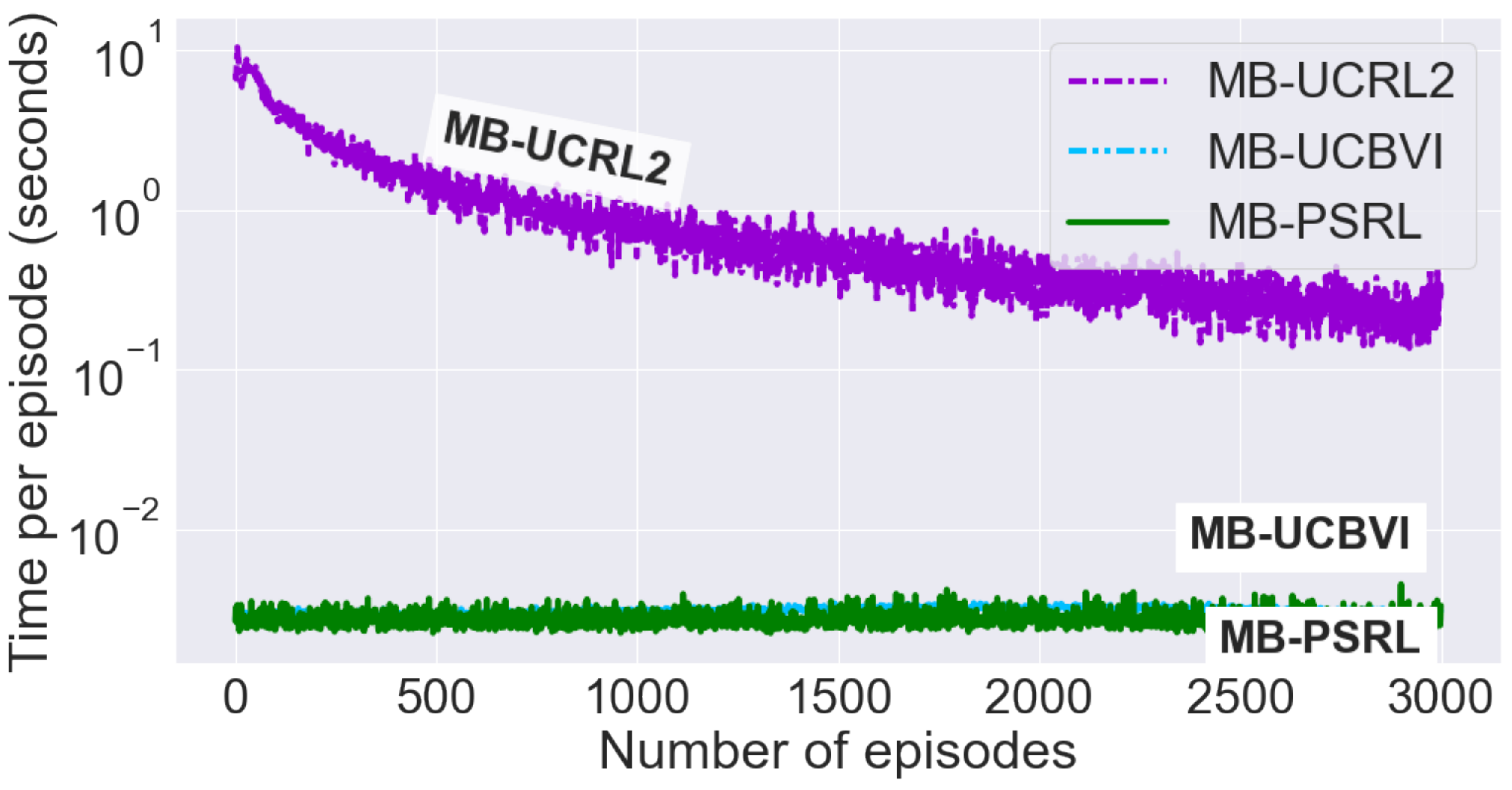}
    }
    \caption{Experimental result for the three 4-state random walk arms given in Table~\ref{fig:randomwalk}. The $x$-axis is the number of episodes. Each algorithm is identified by a unique color for all figures.}
\end{figure}

\paragraph{Robustness (Larger Models and Different Priors)}
To test the robustness of MB-PSRL, we conduct two more sets of experiments that are reported in Appendix~\ref{apx:add_numerical}. They confirm the superiority of MB-PSRL.
The first experiment is an example from \cite{duffQLearningBanditProblems1995a} with $9$ arms each having $11$ states. This model illustrates the effect of the curse of dimensionality: the global  MDP has $11^9$ states which implies that the runtime of MB-UCRL2 makes it impossible to use, while MB-PSRL and MB-UCBVI take a few minutes to complete 3000 episodes.  Also in this example, MB-PSRL seems to converge faster to the optimal policy than MB-UCBVI. The second experiment tests the robustness of MB-PSRL to the choice of prior distribution. We provide numerical evidences that show that, even when MB-PSRL is run with a prior $\phi$ that is not the one from which $M$ is drawn, the regret of MB-PSRL remains acceptable (around twice the regret obtained with a correct prior). 

\section{Conclusion}
\label{sec:conclusion}



In this paper, we present MB-PSRL, a modification of PSRL for Markovian bandit problems. We show that its regret is close to the lower bound that we derive for this problem while its runtime scales linearly with the number of arms. Furthermore, and unlike what is usually the case, MB-PSRL does not have an optimistic counterpart that scales well: we prove that  MB-UCRL2 also has a sub-linear regret but has a computational complexity exponential in the number of arms. This result generalizes to all  the variants of UCRL2 that rely on extended value iteration. We nevertheless show that OFU approach may still be pertinent for Markovian bandit problem:  MB-UCBVI, a version of UCBVI can use Gittins indices and  does not suffer from the dimensionality curse: it has a sub-linear regret in terms of the number of episodes and number of arms as well as a linear time complexity. However its regret remains larger than with MB-PSRL.



\bibliographystyle{splncs04}
\bibliography{refGittinsPS}

\appendix



\vfill
\textbf{All appendix are given in the supplementary material.}

The appendix are organized as follows:
\begin{itemize}
    \item In Appendix~\ref{apx:proof_thm1}, we prove Theorem~\ref{thm:regret_upper_bound}. 
    \item In Appendix~\ref{apx:sketch_of_proof_lower}, we obtain a lower bound of the regret of any reinforcement learning algorithm for Markovian bandits (Theorem~\ref{thm:lower_bound}).
    \item In Appendix~\ref{apx:proof_OFU}, we show that \eqref{eq:EVI} cannot be solved by local indices (Theorem~\ref{thm:no_OFU}).
    \item In Appendix~\ref{apx:algos}, we provide a detailed description of the algorithms that we use in our numerical comparisons. 
    \item In Appendix~\ref{apx:add_numerical}, we provide additional numerical experiments that show the good behavior of MB-PSRL. 
    \item In Appendix~\ref{apx:envi}, we provide details about the experimental environment and the computation time needed. 
\end{itemize}

\newpage


\section{Proof of Theorem~\ref{thm:regret_upper_bound}}
\label{apx:proof_thm1}

The proof of the regret bounds for our three algorithms share a common structure but with different technical details.  In this section, we do a detailed proof of the three algorithms by factorizing as much as possible what can be factorized in the different proofs. This proof is organized as follows:
\begin{itemize}
    \item In Section~\ref{ssec:overview}, we give an overview of the proof that is common to all algorithms. 
    \item In Section~\ref{ssec:technical_lemmas}, we provide technical lemmas that are used in the detailed proofs of each algorithms. 
    \item In Section~\ref{ssec:proof_PSRL}, \ref{ssec:proof_UCRL2} and \ref{ssec:proof_UCBVI}, we provide detailed analysis of MB-PSRL, MB-UCRL2, and MB-UCBVI. 
\end{itemize}

\subsection{Overview of the Proof}
\label{ssec:overview}

Let $\pi_*$ be the optimal policy of the true MDP $M$ and $\pi_k$ the optimal policy for $M_k$, the sampled MDP at episode $k$. Recall that the expected regret is $\sum_{k=1}^K \esp{\Delta_k}$, where $\Delta_k {=} W_{M,1:H_k}^{\pi_*}(\bX_{t_k}) {-} W_{M,1:H_k}^{\pi_k}(\bX_{t_k})$.
For each of the three algorithms, we will define an event $\calE^\text{Algo}_{k-1}$ that is $\calO_{k-1}$-measurable. $\calE^\text{Algo}_{k-1}$ is true with high probability and guarantees that $M$ and $M_k$ are close.
We have:
\begin{align}
    \esp{\Delta_k}
    &= \esp{\Delta_k\ind{\lnot\calE_{k-1}^\text{Algo}} +\Delta_k\ind{\calE_{k-1}^\text{Algo}} } \nonumber\\
    &\le \esp{H_k}\Proba{\lnot\calE_{k-1}^\text{Algo}} + \esp{\Delta_k \ind{\calE_{k-1}^\text{Algo}}} \label{eq:gap_overview}
\end{align}
because $\Delta_k\le H_k$ and the random variables $H_k$ and $\ind{\calE_{k-1}^\text{Algo}}$ are independent.
For each of the three algorithms, the policy $\pi_k$ used at episode $k$ is optimal for a model $M_k$, that is either sampled from the posterior distribution for MB-PSRL, or computed by extended value iteration for MB-UCRL2, or equal to the model with the bonus for MB-UCBVI. We have
\begin{align*}
    \Delta_k=\underbrace{W_{M,1:H_k}^{\pi_*}(\bX_{t_k}) - W_{M_k,1:H_k}^{\pi_k}(\bX_{t_k})}_{:=\Delta^{model}_k}%
    + \underbrace{W_{M_k,1:H_k}^{\pi_k}(\bX_{t_k}) - W_{M,1:H_k}^{\pi_k}(\bX_{t_k})}_{:=\Delta^{conc}_k}.
\end{align*}
As we deal with the expected regret and $H_k$ is independent of the model $M_k$ and of the policy $\pi_k$, we have:
\begin{align}
    \esp{\Delta^{model}_k} = V_{M}^{\pi_*}(\bX_{t_k}) - V_{M_k}^{\pi_k}(\bX_{t_k}) \label{eq:def_model}
\end{align}
As we see later, the above equation can be used to show that $\esp{\Delta^{model}_k \ind{\calE_{k-1}^\text{Algo}}}$ is either $0$ (for MB-PSRL) or non-positive (for MB-UCRL2 or MB-UCBVI). 

We are then left with $\esp{\Delta^{conc}_k \ind{\calE_{k-1}^\text{Algo}}}$. To do so, we use Lemma~\ref{lem:regret_decomposition} to show that there exists a constant $B_k$ (equal to $H_k$ for MB-PSRL and MB-UCRL2 and $H_kL_{k-1}/(2(1-\beta))$ for MB-UCBVI) such that 
\begin{align}
    &\esp{\Delta^{conc}_k\ind{\calE_{k-1}^\text{Algo}}}=\esp{\ind{\calE_{k-1}^\text{Algo}} \Big(W_{M_k,1:H_k}^{\pi_k}(\bX_{t_k}){-}W_{M,1:H_k}^{\pi_k}(\bX_{t_k})\Big)} \nonumber\\
    & {\le} \mathbb{E}\left[\ind{\calE_{k-1}^\text{Algo}} \sum_{t=t_k}^{t_{k+1}{-}1}\abs{r_k(X_{t,A_t}){-}r(X_{t,A_t})}{+}B_k\norm{Q_{k}(X_{t,A_t},\cdot){-}Q(X_{t,A_t},\cdot)}_1\right] \label{eq:proof2}
\end{align}
where ${\norm{Q_{k}(x_a,\cdot)-Q(x_a,\cdot)}_1=\sum_{y_a}\abs{Q_{k}(x_a,y_a)-Q(x_a,y_a)}}$. 
For an arm $a$ and a state $x_a\in\calS^a$, we denote\footnote{In the paper, we use the notation $\ind{E}$ to denote a random variable that equals $1$ if $E$ is true and $0$ otherwise. For instance, $\ind{Y_i=y}=1$ if $Y_i=y$ and $0$ otherwise.} by $N_{k-1}(x_a){=} \sum_{t=1}^{t_{k}-1} \ind{X_{t,A_t}=x_a}$ the number of times that Arm~$a$ is activated before episode $k$ while being in state $x_a$. Equation~\eqref{eq:proof2} relates the performance gap to the distance between the reward functions and transition matrices of the MDPs $M$ and $M_k$. 
With $L_K{=}\sqrt{2\log\frac{4SnK^2\log K}{1-\beta}}$, the event $\calE_{k-1}^\text{Algo}$ guarantees that for all $a, x_a$ and $ k\ge1$, 
\begin{align}
    \label{eq:concentration}
    \abs{r_k(x_a){-}r(x_a)} {\le} \frac{L_K}{\sqrt{\max\{1,N_{k-1}(x_a)\}}} \text{ and }
    \norm{Q_{k}(x_a,\cdot){-}Q(x_a, .)}_1 {\le} \frac{2L_K {+}3\sqrt{S}}{\sqrt{\max\{1,N_{k-1}(x_a)\}}}
\end{align}
We use this with Equation~\eqref{eq:proof2} to show that:
\begin{align}
    \label{eq:proof3}
    \sum_{k=1}^K\esp{\Delta^{conc}_k \ind{\calE^\text{Algo}_{k-1}}} &\le  \esp{C^\text{Algo}_K\sum_{k=1}^K\sum_{t=t_k}^{t_{k+1}-1}\frac{1}{\sqrt{\max\{1,N_{k-1}(x_a)\}}}},
\end{align}
where $C^\text{Algo}_K$ is a random variable that depends on the algorithm studied. 

The final analysis takes care of the right term of \eqref{eq:proof3} and is more technical. It uses the fact that there cannot be too many large terms in this sum because if an arm is activated many times, then $1/\sqrt{N_{k-1}(X_{t,A_t})}$ is small. 
The main technical hurdle here is to deal with the $K$ random episodes $H_1,\ldots, H_K$. 
This is specific to our approach compared to the analysis of finite horizons. 
To bound this, one needs to bound terms of the form $\esp{\max_{1\leq k\leq K} (H_k)^\alpha}$ with $\alpha\in\{1.5,2\}$ (see Equation~\eqref{eq:proof_bound}). 
To bound this, we use the geometric distribution of $H_k$ to show that $\esp{\max_{1\leq k\leq K} (H_k)^{\alpha}}=O((\frac{\log K}{1-\beta})^{\alpha})$ (see Lemma~\ref{lem:moment}).

\subsection{Technical lemmas common to the three algorithms}
\label{ssec:technical_lemmas}

In this section, we establish a series of lemmas that are true for any learning algorithm used. They show that:
\begin{itemize}
    \item The estimates $\hat{r}$ and $\hat{Q}$ concentrates on their true values (Lemma~\ref{lem:concentration};
    \item One can transform $\Delta^{conc}_k$ into Equation~\eqref{eq:proof2} (Lemma~\ref{lem:regret_decomposition});
    \item The sum \eqref{eq:proof3} can be analyzed (Lemma~\ref{lem:sum}).
\end{itemize}

\subsubsection{High Probability Events}

Recall that $\calO_{k-1}$ are the observations collected by the decision maker before episode $k$. Based on $\calO_{k-1}$, we compute the empirical estimators of reward vector and transition matrix as the following: For all $a\in[n]$ and any $x_a\in\calS^a$, let $N_{k-1}(x_a)=\sum_{t=1}^{t_{k}-1}\ind{X_{t,A_t}=x_a}$ be the number of times so far that an arm $a$ was activated in state $x_a$ (at episode $1$, we have ${N_0(x_a)=0}$).
Recall that $t_{k}{:=} 1 {+}\sum_{i=1}^{k-1}H_i$, and that $\hat{r}_{k-1}$ and $\hat{Q}_{k-1}$ are the empirical mean reward vector and transition matrix. More precisely, $\hat{r}_{k-1}(x_a)$ is the empirical mean reward earned when arm $a$ is chosen while being in state $x_a$:
\begin{align*}
    \hat{r}_{k-1}(x_a) = \frac{1}{N_{k-1}(x_a)}\sum_{t=1}^{t_{k}-1} R_t\ind{A_t=a \land X_{t,A_t}=x_a},
\end{align*}
and $\hat{Q}_{k-1}(x_a,y_a)$ is the fraction of times that arm $a$ moved from $x_a$ to $y_a$:
\begin{align*}
    \hat{Q}_{k-1}(x_a,y_a) = \frac{1}{N_{k-1}(x_a)}\sum_{t=1}^{t_{k}-1} \ind{A_t=a \land X_{t,A_t}=x_a \land X_{t+1,A_t}=y_a}.
\end{align*}

We design confidence sets similar to \cite{auerNearoptimalRegretBounds2009,bartlettREGALRegularizationBased2009}.

\begin{lemma}
    \label{lem:concentration}
    For any $k\le K$, let $L_{k-1}=\sqrt{2\log(\frac{2SnK(k-1)\log(K(k-1))}{1-\beta})}$. Let
    \begin{align}
        \calE_{k-1}^H&:= \bigg\{\forall k'\le k{-}1{:} H_{k'}\le \frac{\log(K(k-1))}{1-\beta}\bigg\} \label{eq:conc_h}\\
        \calE_{k-1}^r&:= \bigg\{\forall a\in[n], x_a\in\calS^a, k'\le k{-}1{:} \abs{\rhat_{k'}(x_a){-}r(x_a)} \le \frac{L_{k-1}}{2\sqrt{\max\{1,N_{k'}(x_a)\}}}\bigg\} \label{eq:conc_r} \\
        \calE_{k-1}^Q&:= \bigg\{\forall a\in[n], x_a\in\calS^a, k'\le k{-}1{:} \norm{\Qhat_{k'}(x_a,\cdot){-}Q(x_a, \cdot)}_1 \le \frac{L_{k-1}{+}1.5\sqrt{S}}{\sqrt{\max\{1,N_{k'}(x_a)\}}}\bigg\}. \label{eq:conc_q} \\
        \calE_{k-1}^V&:= \bigg\{\forall a\in[n], \bx\in\mSpace, k'\le k{-}1{:} |\rhat_{k'}(x_a){-}r(x_a) \nonumber \\
        &\qquad +\beta \sum_{\by}(\Phat_{k'}^a(\bx,\by) {-}P^a(\bx,\by))V_M^{\pi_*}(\by)|
        \le \frac{L_{k-1}}{2(1-\beta)\sqrt{\max\{1,N_{k'}(x_a)\}}}\bigg\} \label{eq:conc_V}
    \end{align}
    Then, the above events are all $\calO_{k-1}$-measurable. Moreover:
    \begin{align*}
        &\Proba{\lnot\calE_{k-1}^H}\le1/K \\ 
        &\Proba{\lnot \calE_{k-1}^r }\le 2/K \\
        &\Proba{\lnot \calE_{k-1}^Q }\le 2/K \\
        &\Proba{\lnot \calE_{k-1}^V }\le 2/K.
    \end{align*} 
\end{lemma}

\begin{proof}
    For event $\calE_{k-1}^H$, since $\{H_{k'}\}_{k'\le k{-}1}$ are i.i.d. and geometrically distributed with parameter $(1-\beta)$, we have that
    \begin{align*}
        \proba{\exists k'\le k-1: H_{k'}>\epsilon} \le \sum_{k'=1}^{k-1} \proba{H_{k'}>\epsilon} =(k-1)\beta^{\lfloor \epsilon\rfloor}.
    \end{align*}
    Then, with $\epsilon=\frac{\log(1/(K(k-1)))}{\log(\beta)}$, we get $\proba{\exists k'\le k{-}1: H_{k'}>\epsilon}\le 1/K$.
    Moreover,
    \begin{align*}
        \epsilon = \frac{\log(1/(K(k-1)))}{\log(\beta)}= \frac{\log(K(k-1))}{\log(1/\beta)} < \frac{\log(K(k-1))}{1-\beta}.
    \end{align*}
    Then, $\proba{\exists k'\le k{-}1: H_{k'}>\frac{\log(K(k-1))}{1-\beta}}\le 1/K$.

    Let $\tau_k = (k-1)\frac{\log(K(k-1))}{1-\beta}$. Under event $\calE_{k-1}^H$, the random variable $t_k$ is upper bounded by the deterministic quantity $\tau_k$. In what follows, we assume that event $\calE_{k-1}^H$ holds.

    For event $\calE_{k-1}^r$, let $\tilde{r}_\ell(x_a)$ be a random variable that is the empirical mean of $\ell$ i.i.d. realization of the reward when the arm in state $x_a$ is chosen. In particular, ${\hat{r}_{k-1}(x_a) = \tilde{r}_{N_{k{-}1}(x_a)}(x_a)}$. By Hoeffding's inequality, for any $\epsilon>0$, one has:
    \begin{align*}
        \Proba{\abs{\tilde{r}_\ell(x_a) - r(x_a)} \ge \epsilon  } \le 2 e^{-2 \ell\epsilon^2}.
    \end{align*}
    In particular, this holds for $\epsilon=\sqrt{\frac{\log (2SnK\tau_k)}{2\ell}}$. As $N_{k-1}(x_a)< \tau_k$, by using the union-bound, this implies that: 
    \begin{align}
        &\Proba{\calE_{k-1}^H \land \exists a, x_a, k'\le k{-}1: \abs{\hat{r}_{k'}(x_a) - r(x_a)} \ge \sqrt{\frac{\log (2SnK\tau_k)}{2N_{k'}(x_a)}}}\label{eq:estR_k}\\
        &\qquad\le \sum_{a}\sum_{x_a}\Proba{\exists \ell\in\{1,\dots,\tau_k-1\}: \abs{\tilde{r}_{\ell}(x_a) - r(x_a)} \ge \sqrt{\frac{\log (2SnK\tau_k)}{2\ell}}}\nonumber\\
        &\qquad\le \sum_{\ell=1}^{\tau_k}\sum_a\sum_{x_a}\Proba{\abs{\tilde{r}_{\ell}(x_a) - r(x_a)} \ge \sqrt{\frac{\log (2SnK\tau_k)}{2\ell}}}\nonumber\\
        &\qquad\le  n S \sum_{\ell=1}^{\tau_k} 2e^{-2 \ell \frac{\log (2SnK\tau_k)}{2\ell}} =1/K,\nonumber
    \end{align}
    where the second and third  line is the union on all possible events $N_{k'}(x_a){=}\ell$ for all ${\ell{\in}\{1,\dots, \tau_k-1\}}$.
    In total this says $\proba{\calE_{k-1}^H\land \lnot \calE_{k-1}^r}\le 1/K$.
    Now, $\lnot \calE_{k-1}^r {=}(\calE_{k-1}^H \land \lnot \calE_{k-1}^r) \lor (\lnot \calE_{k-1}^H \land \lnot \calE_{k-1}^r)$.
    Then, using union bound,
    \begin{align*}
        \proba{\lnot \calE_{k-1}^r}
        &\le \proba{\lnot \calE_{k-1}^r \land \calE_{k-1}^H} + \proba{\lnot \calE_{k-1}^r \land \lnot \calE_{k-1}^H} \\
        &\le \proba{\lnot \calE_{k-1}^r \land \calE_{k-1}^H} + \proba{\lnot \calE_{k-1}^H} \le 2/K
    \end{align*}

    The event $\calE_{k-1}^Q$ is similar but by using Weissman's inequality \cite{weissman2003inequalities} instead of Hoeffding's bound. Indeed, by using Equation~(8) in Theorem~2.1 of \cite{weissman2003inequalities}, if $N_{k-1}(x_a)$ was not a random variable, one would have 
    \begin{align*}
        \Proba{\norm{\hat{Q}_{k-1}(x_a,\cdot){-}Q(x_a,\cdot)}_1\ge \epsilon} \le 2^S e^{- N_{k-1}(x_a) \epsilon^2/2}.
    \end{align*}
    Following the same approach as for Equation~\eqref{eq:estR_k} with $\epsilon=\sqrt{2\log(SnK\tau_k2^S)/N_{k-1}(x_a)}$, we use the union-bound to show that:
    \begin{align*}
        &\Proba{\calE_{k-1}^H\land\exists a, x_a, k'{\le}k{-}1: \norm{\hat{Q}_{k'}(x_a,\cdot){-}Q(x_a,\cdot)}_1 {\ge} \sqrt{\frac{2\log(SnK\tau_k 2^S)}{N_{k'}(x_a)}}} \\
        &\qquad \le \tau_k n S 2^Se^{ -N_{k'}(x_a)\frac{2\log(SnK\tau_k2^S)}{2N_{k'}(x_a)}} = 1/K.
    \end{align*}
    By definition of $L_{k-1}=\sqrt{2\log(2SnK\tau_k)}$ and since $\sqrt{x+y}\le\sqrt{x}+\sqrt{y}$, we have
    \begin{align*}
      \sqrt{2\log(SnK\tau_k2^S)}
      &=\sqrt{2\log(2SnK\tau_k) {+} 2(S-1)\log 2}\\
      & \le L_{k-1} +\sqrt{2(S-1)\log2} \le L_{k-1} + 1.5\sqrt{S}.
    \end{align*}
    Hence: 
    \begin{align*}
        &\Proba{\calE_{k-1}^H\land\exists a, x_a, k'\le k{-}1: \norm{\hat{Q}_{k'}(x_a,\cdot)-Q_{}(x_a,\cdot)}_1 \ge \frac{L_{k-1}+1.5\sqrt{S}}{\sqrt{N_{k'}(x_a)}}} \le 1/K.
    \end{align*}
    As done for $\calE^r_{k-1}$, we have $\lnot \calE_{k-1}^Q {=}(\calE_{k-1}^H \land \lnot \calE_{k-1}^Q) \lor (\lnot \calE_{k-1}^H \land \lnot \calE_{k-1}^Q)$. With the same process, we get $\proba{\lnot \calE_{k-1}^Q}\le 2/K$.


    For event $\calE_{k-1}^V$, we have that $\rhat_{k-1}+\Phat_{k-1}V_M^{\pi_*}$ is the empirical mean of $r+PV_M^{\pi_*}$.
    This is because $V_M^{\pi_*}$ is deterministic and $\rhat_{k-1}$ and $\Phat_{k-1}$ are empirical mean of $r$ and $P$ respectively.
    Using Hoeffding's inequality and following the same approach above, we have $\proba{\neg \calE_{k-1}^V}\le 2/K$.
\end{proof}
Note that Lemma~\ref{lem:concentration} is about the statistical properties of the observations $\calO_{k-1}$ in the observation space.
These properties are true for any learning algorithms.
In fact, we will combine different events of this lemma to bound the regret of our algorithm accordingly.

\subsubsection{Concentration Gap}

At episode $k$, our algorithms believe that the unknown MDP $M$ is the MDP $M_k$.
For Bayesian algorithms, $M_k$ is sampled from posterior distribution while for optimistic algorithms, $M_k$ is chosen with respect to optimism principle.
The algorithms follow the policy $\pi_k$ that is optimal for $M_k$.
Recall that $W^{\pi_k}_{M,1:H_k}(\bx)$ is the expected reward of the MDP $M$ under policy $\pi_k$, starts in state $\bx$ and lasts for $H_k$ time steps and the expected cumulative discounted reward in $M$ starting from state $\bx$ under policy $\pi_k$ is $V_{M}^{\pi_k}(\bx){=}\mathbb{E}[W^{\pi_k}_{M, 1:H_k}(\bx)]$ where $H_k\sim \mathrm{Geom}(1-\beta)$ is the horizon of episode $k$.

\begin{lemma}
    \label{lem:regret_decomposition}
    For episode $k$, let $B_k\in\R^+$ be an upper bound\footnote{We will use $B_k=H_k$ for MB-PSRL and MB-UCRL2 and $B_k=H_kL_{k-1}/(2(1-\beta))$ for MB-UCBVI.} of $W_{M_k,1:H_k}^{\pi_k}(\bx)$, i.e., a constant $B_k$ such that for any $\bx\in\mSpace$, $W_{M_k,1:H_k}^{\pi_k}(\bx)\le B_k$. We have,
    \begin{align}
        \label{eq:lem_decompo}
        &\esp{ \Delta_k^{conc} {\mid} \calO_{k-1}, H_k, M_k, M} {=} \esp{ W_{M_k,1:H_k}^{\pi_k}(\bX_{t_k}){-}W_{M,1:H_k}^{\pi_k}(\bX_{t_k}) {\mid} \calO_{k-1}, H_k, M_k, M}\nonumber \\
        &{\le}\mathbb{E}\bigg[\sum_{t=t_k}^{t_{k+1}-1}\abs{r_k(X_{t,A_t}){-}r(X_{t,A_t})}
        {+}B_k\norm{Q_{k}(X_{t,A_t},\cdot){-}Q(X_{t,A_t},\cdot)}_1 {\mid} \calO_{k-1}, H_k, M_k, M\bigg]
    \end{align}
\end{lemma}

\begin{proof}
    From \eqref{eq:discounted_reward} with $a=\pi_k(\bx)$,
    \begin{align}
        W^{\pi_k}_{M, 1:H_k}(\bx) 
        &=r(x_a) + \sum_{\by} P^{\pi_k}(\bx, \by)W^{\pi_k}_{M,2:H_k}(\by) \label{eq:decompo2}
    \end{align}
    where $P^{\pi_k}$ is the state transition dynamic of the system when following the policy $\pi_k$. Comparing the sampled MDP $M_k$ with the original $M$ and using \eqref{eq:decompo2}, one has
    \begin{align*}
        W^{\pi_k}_{M_k, 1:H_k}(\bx) {-} W^{\pi_k}_{M, 1:H_k}(\bx)
        &= r_k(x_a) {-} r(x_a) \\
        &{+} \sum_{\by} P_k^{\pi_k}(\bx, \by)W^{\pi_k}_{M_k, 2:H_k}(\by) {-} \sum_{\by} P^{\pi_k}(\bx, \by)W^{\pi_k}_{M, 2:H_k}(\by).
    \end{align*}
    Note that in the above equation, the last term is of the form $P_k^{\pi_k}W^{\pi_k}_{M_k} {-}P^{\pi_k}W^{\pi_k}_{M}$, which is equal to $(P_k^{\pi_k} {-}P^{\pi_k})W^{\pi_k}_{M_k} {+}P^{\pi_k}(W^{\pi_k}_{M_k} {-}W^{\pi_k}_{M})$. Moreover, $W_{M_k}^{\pi_k}$ is less than $B_k$.
    Plugging this to the above equation shows that:
    \begin{align*}
        &W^{\pi_k}_{M_k, 1:H_k}(\bx) {-} W^{\pi_k}_{M, 1:H_k}(\bx) \\
        &\quad\le \abs{r_k(x_a) {-} r(x_a)} {+} B_k\sum_{\by}\abs{P_k^{\pi_k}(\bx, \by) {-} P^{\pi_k}(\bx, \by)}\\
        &\qquad {+} \sum_{\by} P^{\pi_k}(\bx, \by)(W^{\pi_k}_{M_k, 2:H_k}(\by) {-} W^{\pi_k}_{M, 2:H_k}(\by))\\
        &\quad {=}\abs{r_k(x_a) {-} r(x_a)} {+} B_k\norm{P_k^{\pi_k}(\bx,\cdot){-}P^{\pi_k}(\bx,\cdot)}_1 {+} D^{M_k,M}_{H_k}(\bx) \\
        &\qquad {+}W^{\pi_k}_{M_k, 2:H_k}(\bX_1) {-} W^{\pi_k}_{M, 2:H_k}(\bX_1)
    \end{align*}
    where $D^{M_k,M}_{H_k}(\bx){:=} \sum_{\by}P^{\pi_k}(\bx, \by)(W^{\pi_k}_{M_k, 2:H_k}(\by) {-} W^{\pi_k}_{M, 2:H_k}(\by)) {-}(W^{\pi_k}_{M_k, 2:H_k}(\bX_1) {-} W^{\pi_k}_{M, 2:H_k}(\bX_1))$.
    Note that in the equation above, $D^{M_k,M}_{H_k}(\bx)$ is a martingale difference with ${\bX_1\sim P^{\pi_k}(\bx,\cdot)}$.
    Hence, the expected value of the martingale difference sequence  is zero.
    As only arm $a$ makes a transition, we have $\norm{P_k^{\pi_k}(\bx,\cdot)-P^{\pi_k}(\bx,\cdot)}_1=\norm{Q_{k}(x_a,\cdot)-Q(x_a,\cdot)}_1$. Hence, a direct induction shows that \eqref{eq:lem_decompo} holds.
\end{proof}

\subsubsection{Bound on  the double sum}

Recall that for $k\le K$, any $a\in[n]$ and any $x_a\in\calS^a$, $N_{k-1}(x_a)=\sum_{t=1}^{t_{k}-1}\ind{X_{t,A_t}=x_a}$ is the number of times so far that an arm $a$ was activated in state $x_a$ (at episode $1$, we have ${N_0(x_a)=0}$) and $\{H_k\}_{k\le K}$ be the sequence of episode horizons.

\begin{lemma}
    \label{lem:sum}
    For any learning algorithms, we have
    \begin{align*}
        \sum_{k=1}^K \sum_{t=t_k}^{t_{k+1}-1}\frac1{\sqrt{\max\{1,N_{k-1}(X_{t,A_t})\} }} \le Sn\max_{k\le K}H_k+2\sqrt{SnK\max_{k\le K}H_k}
    \end{align*}
\end{lemma}
\begin{proof}
Let $\tilde{N}_t(x_a)$ be the number of times that arm $a$ has been activated before time $t$ while being in state $x_a$. By definition, $\tilde{N}_{t_k}(x_a)=N_{k-1}(x_a)$. Moreover, if $t\in\{t_k,\dots, t_{k+1}-1\}$, then $\tilde{N}_t(x_a)\le N_{k-1}(x_a) + H_k$. This shows that 
\begin{align*}
    \sum_{k=1}^K \sum_{t=t_k}^{t_{k+1}-1}\frac1{\sqrt{\max\{1,N_{k-1}(X_{t,A_t})\} }}
    &\le \sum_{k=1}^K \sum_{t=t_k}^{t_{k+1}-1}\frac{1}{\sqrt{\max\{1, \tilde{N}_t(X_{t,A_t})-H_k\} }}\\
    &\le \sum_{t=1}^{t_{K+1}-1} \frac{1}{\sqrt{\max\{1,\tilde{N}_t(X_{t,A_t})-\max_k H_k\} }}.
\end{align*}
The above sum can be reordered to group terms by state: The above sum equals
\begin{align*}        
    \sum_{a, x_a} \sum_{m=1}^{\tilde{N}_{t_{K+1}}(x_a)} \frac{1}{\sqrt{\max\{1, m-\max_k H_k\} }}
    &\le \sum_{a, x_a} \left[\max_k H_k + \sum_{m=1}^{\max\{1,\tilde{N}_{t_{K+1}}(x_a)-\max_k H_k\}} \frac{1}{\sqrt{m}}\right],\\
    &\le Sn\max_k H_k + \sum_{a, x_a} \sum_{m=1}^{\tilde{N}_{t_{K+1}}(x_a)} \frac{1}{\sqrt{m}},\\
    &\le Sn\max_k H_k + 2\sum_{a, x_a} \sqrt{\tilde{N}_{t_{K+1}}(x_a)},
\end{align*}
where the last inequality holds because $\sum_{m=1}^{t_{K+1}}1/\sqrt{m}\le\int_1^{t_{K+1}}1/\sqrt{x}dx\le2\sqrt{{t_{K+1}}}$. 

Now, by Cauchy-Schwartz inequality, and because $\sum_{a,x_a}\tilde{N}_{t_{K+1}}(x_a) =t_{K+1}-1 {=}\sum_{k=1}^{K}H_k$, we have:
\begin{align*}
    \sum_{a,x_a}\sqrt{\tilde{N}_{t_{K+1}}(x_a)} \le \Big(\sum_{a,x_a}\tilde{N}_{t_{K+1}}(x_a)\Big)^{1/2}\Big(\sum_{a,x_a}1\Big)^{1/2}=\sqrt{Sn\sum_{k=1}^{K}H_k}\le \sqrt{SnK \max_{k\le K}H_k}.
\end{align*}
\end{proof}

\subsubsection{Bound on the expectation of $\esp{\max_{k\le K}H_k}$}

\begin{lemma}
    \label{lem:moment}
    Let $\alpha\in[1,2.5]$. Then,
    \begin{align}
        \label{eq:lemma_HK}
        \esp{\max_{k\le K}(H_k)^\alpha} &\le 5+5\left(\frac{\log K}{1-\beta}\right)^\alpha.
    \end{align}
 \end{lemma}
 \begin{proof}
    By definition, we have
    \begin{align*}
        \esp{\max_{k\le K}(H_k)^\alpha} &=  \sum_{i=1}^\infty \Proba{ \max_{k\le K}(H_k)^\alpha \ge i}\\
        &\le \sum_{i=1}^\infty \min(1, K \Proba{ (H_k)^\alpha \ge i})\\
        &= \sum_{i=1}^\infty \min(1, K \beta^{i^{1/\alpha}}),
    \end{align*}
    where the inequality comes from the union bound and the last equality is because the random variables $H_k$ are geometrically distributed.
 
    Let $A=\min\{i : K \beta^{i^{1/\alpha}}\le 1\}$. Decomposing the above sum by group of size $A$, we have
    \begin{align}
        \sum_{i=1}^\infty \min(1, K \beta^{i^{1/\alpha}})
        &= \sum_{j=0}^\infty \sum_{i=Aj+1}^{A(j+1)}\min(1, K \beta^{i^{1/\alpha}})\nonumber\\
        &\le \sum_{j=0}^\infty A \min(1, K \beta^{(Aj)^{1/\alpha}})\nonumber\\
        &= A + A\sum_{j=1}^\infty K (\beta^{A^{1/\alpha}})^{j^{1/\alpha}},
        \label{eq:sum_A}
    \end{align}
    where the inequality holds because $\beta^{i^{1/\alpha}}$ is decreasing in $i$. 
 
    By definition of $A$, we have $\beta^{A^{1/\alpha}}\le 1/K$. This implies that the second term of Equation~\eqref{eq:sum_A} is smaller than $\sum_{j=1}^\infty K (1/K)^{j^{1/\alpha}}=\sum_{j=1}^\infty K^{1-j^{1/\alpha}}$. As $\alpha\le 2.5$, if  $K\ge5$, this is smaller than $\sum_{j=1}^\infty 5^{1-j^{1/2.5}}\approx 3.92<4$.
 
    This shows that for $K\ge5$, we have:
    \begin{align*}
        \esp{\max_{k\le K}(H_k)^\alpha} \le 5A,
    \end{align*}
    where $A=\ceil{(-\log K/\log \beta)^{\alpha}}\le 1+(\log K/(1-\beta ))^\alpha$.
 
    As for the case where $K\le4$, we have $\esp{\max_{k\le K}(H_k)^\alpha}\le K\esp{H_1^\alpha}\le \frac{K}{(1-\beta)^{\alpha}}$. This term is smaller than \eqref{eq:lemma_HK} for $K\le4$.
 \end{proof}

\subsection{Detailed analysis of MB-PSRL}
\label{ssec:proof_PSRL}

We decompose the analysis of PSRL in three steps: 
\begin{itemize}
    \item We define the high-probability event $\calE^\text{PSRL}_{k-1}$.
    \item We analyze $\sum_{k=1}^K\esp{\Delta^{model}_k \ind{\calE^\text{PSRL}_{k-1}}}$ (which equals $0$ here because of posterior sampling).
    \item We analyze $\sum_{k=1}^K\esp{\Delta^{conc}_k \ind{\calE^\text{PSRL}_{k-1}}}$. 
\end{itemize}
We will use the same proof structure for MB-UCRL2 and MB-UCBVI. 

Before doing the proof, we start by a first lemma that that essentially formalizes the fact that the distribution of $M$ given $\calO_{k-1}$ is the same as the distribution of the sampled MDP $M_{k}$ conditioned on $\calO_{k-1}$. 
\begin{lemma}
    \label{lem:bayesian}
    Assume that the MDP $M$ is drawn according to the prior $\phi$ and that $M_k$ is draw according to the posterior $\phi(\cdot \mid \calO_{k-1})$. Then, for any $\calO_{k-1}$-measurable function $g$, one has:
    \begin{align}
        \label{eq:lem_bayesian}
        \esp{g(M)} = \esp{g(M_k)}. 
    \end{align}
\end{lemma}
\begin{proof}
    At the start of each episode $k$, MB-PSRL computes the posterior distribution of $M$ conditioned on the observations $\calO_{k-1}$, and draws $M_k$ from it. This implies that $M$ and $M_k$ are identically distributed conditioned on $\calO_{k-1}$. Consequently, if $g$ is a $\calO_{k-1}$-measurable function, one has:
    \begin{align*}
        \esp{g(M) \mid \calO_{k-1}} = \esp{g(M_k) \mid \calO_{k-1}}. 
    \end{align*}
    Equation~\eqref{eq:lem_bayesian} then follows from the tower rule. 
\end{proof}

\subsubsection{Definition of the high-probability event $\calE^\text{PSRL}_{k-1}$}

\begin{lemma}
    \label{lem:concentration_psrl}
    At episode $k$, the event
    \begin{align*}
        \calE^\text{PSRL}_{k-1} = &\bigg\{ \forall a{\in}[n], x_a{\in}\calS^a, k'\le k{-}1{:}
            \abs{r_{k'+1}(x_a){-}r(x_a)} \le\frac{L_{k-1}}{\sqrt{\max\{1,N_{k'}(x_a)\} }}, \\
            &\quad\norm{Q_{k'+1}(x_a,\cdot) {-}Q(x_a,\cdot)}_1 \le\frac{2L_{k-1}{+}3\sqrt{S}}{\sqrt{\max\{1,N_{k'}(x_a)\} }}, \text{ and } H_{k'}\le \frac{\log(K(k-1))}{1-\beta} \bigg\}
    \end{align*}
    is $\calO_{k-1}$-measurable and true with probability at least $1-9/K$.
\end{lemma}
 
\begin{proof}
    Recall that for MB-PSRL, at the beginning of episode $k$, we sample a MDP $M_{k}$. We define the two events that are the analogue of the events \eqref{eq:conc_r} and \eqref{eq:conc_q} of Lemma~\ref{lem:concentration} but replacing the true MDP $M$ by the sampled MDP $M_{k}$: 
    \begin{align*}
        &\tilde{\calE}_{k-1}^r:= \bigg\{\forall a\in[n], x_a\in\calS^a, k'\le k{-}1{:} \abs{\rhat_{k'}(x_a){-}r_{k'+1}(x_a)} \le \frac{L_{k-1}}{2\sqrt{\max\{1,N_{k'}(x_a)\} }}\bigg\} \\
        &\tilde{\calE}_{k-1}^Q:= \bigg\{\forall a\in[n], x_a\in\calS^a, k'\le k{-}1{:} \norm{\Qhat_{k'}(x_a,\cdot){-}Q_{k'+1}(x_a, \cdot)}_1 \le \frac{L_{k-1}{+}1.5\sqrt{S}}{\sqrt{\max\{1,N_{k'}(x_a)\} }}\bigg\} 
    \end{align*}
    These events are $\calO_{k-1}$-measurable. Hence, Lemma~\ref{lem:bayesian}, combined with Lemma~\ref{lem:concentration} implies that $\Proba{ \lnot \tilde{\calE}_{k-1}^r} =\Proba{ \lnot \calE_{k-1}^r}\le 2/K$ and $\Proba{\lnot \tilde{\calE}_{k-1}^Q}=\Proba{ \lnot \calE_{k-1}^Q}\le 2/K$. Since the complement of $\calE_{k-1}^\text{PSRL}$ is the union of $\lnot \calE_{k-1}^r, \lnot \tilde{\calE}_{k-1}^r, \lnot \calE_{k-1}^Q, \lnot \tilde{\calE}_{k-1}^Q$ and $\lnot \calE_{k-1}^H$, the union bound implies that $\Proba{\calE^\text{PSRL}_{k-1}}\ge1-9/K$.

\end{proof}

\subsubsection{Analysis of $\esp{\Delta^{model}_k \ind{\calE_{k-1}^\text{PSRL}}}$ for MB-PSRL.}

Lemma~\ref{lem:bayesian} implies that for MB-PSRL, $\esp{\Delta^{model}_k \ind{\calE_{k-1}^\text{PSRL}}}{=}0$ because
$\calE_{k-1}^\text{PSRL}$, $\pi_k$ and $M_k$ are $\calO_{k-1}$-measurable. 

\subsubsection{Analysis of $\esp{\Delta^{conc}_k}$ for MB-PSRL.}
Following \eqref{eq:gap_overview}, the Bayesian regret can be written as:
\begin{align}
    \BayReg(K,\text{MB-PSRL},\phi)
    &= \sum_{k=1}^K \esp{\Delta_k} \le \sum_{k=1}^{K}\esp{H_k}\proba{\neg \calE^\text{PSRL}_{k-1}} {+}\esp{\Delta_k \ind{\calE^\text{PSRL}_{k-1}}} \nonumber \\
    &\le \frac{9}{(1-\beta)} {+} \sum_{k=1}^{K} \esp{\Delta^{model}_k\ind{\calE^\text{PSRL}_{k-1}}} {+} \esp{\Delta^{conc}_k\ind{\calE^\text{PSRL}_{k-1}}}
    \label{eq:proof_PSRL_regret}
\end{align}
where the last inequality holds due to Lemma~\ref{lem:concentration_psrl}.
By the previous section, the second term of \eqref{eq:proof_PSRL_regret} is zero.
As all rewards are bounded by $1$, $W_{M_k,1:H_k}^{\pi_k}(\bX_{t_k})\le H_k$. Hence, by applying Lemma~\ref{lem:regret_decomposition} with the upper bound $B_k=H_k$, and because $\ind{\calE_{k-1}^\text{PSRL}}$ is deterministic given $\calO_{k-1}$, we have
\begin{align}
    \esp{\Delta_k^{conc} \ind{\calE_{k-1}^\text{PSRL}}}
    &= \esp{\esp{\Delta_k^{conc} \ind{\calE_{k-1}^\text{PSRL}} \mid \calO_{k-1}, H_k, M_k, M}} \nonumber \\
    &\le \mathbb{E}\bigg[ \ind{\calE_{k-1}^\text{PSRL}} \sum_{t=t_k}^{t_{k+1}-1}\abs{r_k(X_{t,A_t}){-}r(X_{t,A_t})} \nonumber\\
    &\qquad +H_k\norm{Q_{k}(X_{t,A_t},\cdot){-}Q(X_{t,A_t},\cdot)}_1 \bigg]. \label{eq:conc_with_i}
\end{align}
Let $\calR_k := \sum_{t=t_k}^{t_{k+1}-1}\abs{r_k(X_{t,A_t}){-}r(X_{t,A_t})} {+}H_k\norm{Q_{k}(X_{t,A_t},\cdot){-}Q(X_{t,A_t},\cdot)}_1$. 
By using the definition of $\calE_{k-1}^\text{PSRL}$, we have:
\begin{align}
    \ind{\calE_{k-1}^\text{PSRL}}\calR_k 
    &\le \ind{\calE_{k-1}^\text{PSRL}} \sum_{t=t_k}^{t_{k+1}-1} \frac{L_{k-1}{+}(2L_{k-1} {+}3\sqrt{S})H_k}{\sqrt{\max\{1,N_{k-1}(X_{t,A_t})\} }} \nonumber\\
    &\le \sum_{t=t_k}^{t_{k+1}-1} \frac{L_{k-1}{+}(2L_{k-1} {+}3\sqrt{S})H_k}{\sqrt{\max\{1,N_{k-1}(X_{t,A_t})\} }} \label{eq:PSRL_R_k}
\end{align}
Hence, summing over all $K$ episodes gives us:
\begin{align}
    \sum_{k=1}^{K} \ind{\calE_{k-1}^\text{PSRL}}\calR_k
    &\le  \big(L_K {+}(2L_K {+}3\sqrt{S})\max_{k\le K}H_k\big)\sum_{k=1}^K \sum_{t=t_k}^{t_{k+1}-1}\frac1{\sqrt{\max\{1,N_{k-1}(X_{t,A_t})\} }}
    \nonumber\\
    &\le  3(L_K+\sqrt{S})\max_{k\le K}H_k \sum_{k=1}^K \sum_{t=t_k}^{t_{k+1}-1}\frac1{\sqrt{\max\{1,N_{k-1}(X_{t,A_t})\} }},
    \label{eq:sum_not_E_k}
\end{align}
where the first inequality holds because $L_k\le L_K$ and $\max_{k\le K}H_k\ge1$. Note that the last inequality leads to a slightly worst bound but simplifies the expression. By Lemma~\ref{lem:sum}, we get
\begin{align*}
    \sum_{k=1}^{K} \ind{\calE_{k-1}^\text{PSRL}}\calR_k
    &\le 3(L_K+\sqrt{S})\max_{k\le K}H_k(Sn\max_{k\le K}H_k +2\sqrt{SnK \max_{k\le K}H_k}) \\
    &=3(L_K+\sqrt{S})(Sn\max_{k\le K}(H_k)^2 +2\sqrt{SnK} \max_{k\le K}(H_k)^{3/2})
    \nonumber
\end{align*}
Then, 
\begin{align}
    \sum_{k=1}^K\esp{\Delta_k^{conc} \ind{\calE_{k-1}^\text{PSRL}}}
    &{\le} 3(L_K {+}\sqrt{S})\bigg(Sn\esp{\max_{k\le K}(H_k)^2} {+}2\sqrt{SnK} \esp{\max_{k\le K}(H_k)^{3/2}}\bigg) \label{eq:proof_bound}\\
    &{\le} 3(L_K {+}\sqrt{S})\bigg(Sn\p{5 {+}5\bigg(\frac{\log K}{1-\beta}\bigg)}^2 {+}\sqrt{SnK}\p{5 {+}5\bigg(\frac{\log K}{1-\beta}\bigg)}^{3/2}\bigg) \nonumber
\end{align}
where the last inequality is true due to Lemma~\ref{lem:moment}.
With $L_K{=}\sqrt{2\log\frac{4SnK^2\log K}{1-\beta}}$, this implies that there exists a constant $C$ independent of all problem's parameters such that:
\begin{align*}
    \BayReg(K,\text{MB-PSRL},\phi) {\le} C\p{\sqrt{S}{+}\log\p{\frac{SnK\log K}{1-\beta}}} \p{Sn\p{\frac{\log K}{1-\beta}}^{2}
    {+}\sqrt{SnK}\p{\frac{\log K}{1-\beta}}^{3/2} }.
\end{align*}

\subsubsection{Remark on the dependence on $S$}

Our bound is linear in $S$, the state size of each arm, because our proof follows the approach used in \cite{osbandMoreEfficientReinforcement2013a}. Using another proof methodology, it is argued in \cite{osband2017posterior} that the regret of PSRL grows as the square root of the state space size and not linearly. In our paper, we choose to use the more conservative approach of \cite{osbandMoreEfficientReinforcement2013a} because we believe that the proof used in \cite{osband2017posterior} is not correct (in particular the use of a deterministic $v$ in Equation~(16) of the proof of Lemma~3 in Appendix~A in the arXiv version of \cite{osband2017posterior} seems incompatible with the use of Lemma~4 of the same paper).
In fact, when considering the worst case realization of $v$, the concentration bound in Equation~(16) of the paper is equivalent to the (scaled) L1 norm of transition concentration.
We are not alone to point out this error. Effectively, \cite{agrawal17} used Lemma~C.1 and Lemma~C.3 (equivalence of Lemma~3 of \cite{osband2017posterior}) to get a bound in square root of the state space size. But both lemmas are erroneous as mentioned in the latest arXiv version of \cite{agrawal17}. The validity of Lemma~3 is also questioned on page 87 of \cite{fruit2019}. While it is informal, the recent work of \cite{concentration_ineq_multinoulli} also theoretically contradicts the lemma.



\subsection{Case of MB-UCRL2}
\label{ssec:proof_UCRL2}

The proof follows the same steps as for MB-PSRL. While the high probability event is simpler, the additional complexity is to show that $\sum_{k=1}^K\esp{\Delta^{model}_k}\le0$ by using the optimism principle. 

\subsubsection{Definition of the high probability event} 

\begin{lemma}
    \label{lem:concentration_ucrl}
    At episode $k$, the event
    \begin{align*}
        \calE^\text{UCRL2}_{k-1} =&\bigg\{ \forall a{\in}[n], x_a{\in}\calS^a, k'\le k{-}1{:}
            \abs{\rhat_{k'}(x_a){-}r(x_a)} \le\frac{L_{k-1}}{2\sqrt{\max\{1,N_{k'}(x_a)\} }}, \\
            &\qquad\norm{\Qhat_{k'}(x_a,\cdot) {-}Q(x_a,\cdot)}_1 \le\frac{L_{k-1}{+}1.5\sqrt{S}}{\sqrt{\max\{1,N_{k'}(x_a)\} }}, \text{ and } H_{k'}\le \frac{\log(K(k-1))}{1-\beta} \bigg\}
    \end{align*}
    is $\calO_{k-1}$-measurable and true with probability at least $1-5/K$.
\end{lemma}
\begin{proof}
    The complement of $\calE_{k-1}^\text{UCRL2}$ is the union of $\neg \calE_{k-1}^r, \neg \calE_{k-1}^Q$ and $\neg \calE_{k-1}^H$. We conclude the proof by using the union bound and $\proba{\neg \calE_{k-1}^r} \le 2/K$, $\proba{\neg \calE_{k-1}^Q} \le 2/K$ and $\proba{\neg \calE_{k-1}^H} \le 1/K$. 
\end{proof}

\subsubsection{Analysis of $\esp{\Delta^{model}_k\ind{\calE_{k-1}^\text{UCRL2}}}$ -- Optimism of MB-UCRL2}

Recall that $\pi_*$ is the optimal policy of the unknown MDP $M$ and that $\pi_k$ is the policy used in episode $k$. $\pi_k$ is optimal for the optimistic MDP that is chosen from the plausible MDP set $\M_k$:
\begin{align*}
    \pi_k \in \argmax_{\pi} \max_{M'\in\M_k} V^\pi_{M'}.
\end{align*}
For each episode $k$, the plausible MDP set $\M_k$ is defined by
\begin{align}
    \label{eq:plausible}
    \M_k=\bigg\{(r',Q'):\forall a, x_a, \abs{r'(x_a)-\hat{r}_{k-1}(x_a)} \le \frac{L_{k-1}}{2\sqrt{\max\{1,N_{k-1}(x_a)\} }} \text{, and }\nonumber\\
    \norm{Q'(x_a,\cdot)-\hat{Q}_{k-1}(x_a, .)}_1 \le \frac{L_{k-1}+1.5\sqrt{S}}{\sqrt{\max\{1,N_{k-1}(x_a)\} }}\bigg\}.
\end{align}
As \cite{auerNearoptimalRegretBounds2009}, we argue that there exists a MDP $M_k\in\M_k$ such that $\pi_k$ is an optimal policy for $M_k$. Moreover, under event $\calE^\text{UCRL2}_{k-1}$, one has $M\in \M_k$, which implies that $\max_{\pi}\max_{M'\in \M_k} V^{\pi}_{M'}(\bx)\ge V^{\pi_*}_M(\bx)$. By \eqref{eq:def_model}, we get $\esp{\Delta^{model}_k}{\le} 0$.
If $\calE^\text{UCRL2}_{k-1}$ does not hold, we simply have $\Delta^{model}_k\ind{\calE_{k-1}^\text{UCRL2}}=0$. We conclude that: $\esp{\Delta^{model}_k\ind{\calE^\text{UCRL2}_{k-1}}}{\le}0$.

\subsubsection{Analysis of $\esp{\Delta^{conc}_k\ind{\calE^\text{UCRL2}_{k-1}}}$ for MB-UCRL2}
Following \eqref{eq:gap_overview}, the expected regret can be written as:
\begin{align}
    \esp{\Reg(K,\text{MB-UCRL2},M)}
    &= \sum_{k=1}^K \esp{\Delta_k} \le \sum_{k=1}^{K}\esp{H_k}\proba{\neg \calE^\text{UCRL2}_{k-1}} {+}\esp{\Delta_k \ind{\calE^\text{UCRL2}_{k-1}}} \nonumber \\
    &\le \frac{5}{1-\beta} {+} \sum_{k=1}^{K} \esp{\Delta^{model}_k\ind{\lnot\calE^\text{UCRL2}_{k-1}}} {+} \esp{\Delta^{conc}_k\ind{\lnot\calE^\text{UCRL2}_{k-1}}}
    \label{eq:proof_UCRL2_regret}
\end{align}
where the last inequality holds due to Lemma~\ref{lem:concentration_ucrl}.
By the previous section, the second term of \eqref{eq:proof_UCRL2_regret} is non-positive. 
In the following, we therefore analyze the last term whose analysis is then similar to the one for MB-PSRL.  Indeed, with $B_k=H_k$ and definition of $\calE_{k-1}^\text{UCRL2}$, the use of Lemma~\ref{lem:regret_decomposition} shows that one has
\begin{align*}
    \esp{\Delta^\text{conc}_k\ind{\calE_{k-1}^\text{UCRL2}}} \le \esp{\frac12 \sum_{t=t_k}^{t_{k+1}-1}\frac{L_{k-1} {+}(2L_{k-1} {+} 3\sqrt{S}) H_k}{\sqrt{\max\{1,N_{k-1}(X_{t,A_t})\} }}}.
\end{align*}
Up to a factor $1/2$, the expression inside the expectation is the same as Equation~\eqref{eq:PSRL_R_k} of MB-PSRL. Hence, one can use Lemma~\ref{lem:sum} the same way to show that
\begin{align*}
    \sum_{k=1}^{K} \esp{\Delta_k^{conc}\ind{\calE_{k-1}^\text{PSRL}}}
    \le \frac32(L_K+\sqrt{S})\bigg(Sn\esp{\max_{k\le K}H_k^2} +2\sqrt{SnK} \esp{\max_{k\le K}H_k^{3/2}}\bigg).
\end{align*}
Up to a factor $1/2$, the right term of the above equation is equal to the right term of \eqref{eq:proof_bound}. Following the same process done for the later, we can conclude that there exists a constant $C'$ independent of all problem's parameters such that:
\begin{align*}
    \Reg(K,\text{MB-UCRL2},M) \le C'\p{\sqrt{S}{+}\log\p{\frac{SnK\log K}{1-\beta}}} \p{Sn\p{\frac{\log K}{1-\beta}}^{2}
    {+}\sqrt{SnK}\p{\frac{\log K}{1-\beta}}^{3/2} }
\end{align*}

\subsection{Case of MB-UCBVI}
\label{ssec:proof_UCBVI}

We start by defining the high probability event. Then, we prove the optimistic property of MB-UCBVI. Finally, we bound its expected regret.

\subsubsection{Definition of the high-probability event}

\begin{lemma}
    \label{lem:concentration_ucbvi}
    The event
    \begin{align*}
        \calE^\text{UCBVI}_{k-1} {=}\bigg\{ \forall a{\in}[n], \bx{\in}\mSpace, k'\le k{-}1{:}
            \abs{\rhat_{k'}(x_a){-}r(x_a)} \le\frac{L_{k-1}}{2\sqrt{\max\{1,N_{k'}(x_a)\} }}, \\
            \norm{\Qhat_{k'}(x_a,\cdot) {-}Q(x_a,\cdot)}_1 \le\frac{L_{k-1}{+}1.5\sqrt{S}}{\sqrt{\max\{1,N_{k'}(x_a)\} }}, H_{k'}\le \frac{\log(K(k-1))}{1-\beta}, \\
        \text{ and } \abs{\rhat_{k'}(x_a){-}r(x_a) {+}\beta\sum_{\by}(\Phat_{k'}^a(\bx,\by){-}P^a(\bx,\by))V_M^{\pi_*}(\by)} \le\frac{L_{k-1}}{2(1{-}\beta)\sqrt{\max\{1,N_{k'}(x_a)\} }} \bigg\}
    \end{align*}
    is $\calO_{k-1}$-measurable and true with probability at least $1-7/K$.
\end{lemma}
\begin{proof}
    The complement of $\calE_{k-1}^\text{UCBVI}$ is the union of $\neg \calE_{k-1}^r, \neg \calE_{k-1}^Q, \neg \calE_{k-1}^H$ and $\neg \calE_{k-1}^V$.
    We conclude the proof by using the union bound and $\proba{\neg \calE_{k-1}^r} \le 2/K$, $\proba{\neg \calE_{k-1}^Q} \le 2/K$, $\proba{\neg \calE_{k-1}^V} \le 2/K$ and $\proba{\neg \calE_{k-1}^H} \le 1/K$
\end{proof}

\subsubsection{Analysis of $\esp{\Delta^{model}_k \ind{\calE_{k-1}^\text{UCBVI}}}$ -- Optimism of MB-UCBVI}

The following lemma guarantees that $\esp{\Delta^{model}_k \ind{\calE_{k-1}^\text{UCBVI}}}\le0$. Indeed, as $\calE_{k-1}^\text{UCBVI}$ is $\calO_{k-1}$-measurable, one has 
\begin{align*}
    \esp{\Delta^{model}_k\ind{\calE_{k-1}^\text{UCBVI}}} &= \esp{\esp{\Delta^{model}_k\mid \calO_{k-1}}\ind{\calE_{k-1}^\text{UCBVI}}}\\
    &=\esp{(V_M^{\pi_*}(\bX_{t_k}) - V_{M_k}^{\pi_k}(\bX_{t_k}))\ind{\calE_{k-1}^\text{UCBVI}}} \le 0.
\end{align*}
\begin{lemma}
    \label{lem:ucbvi_optim}
   If $\calE_{k-1}^\text{UCBVI}$ is true, then, for any $\bx\in\mSpace$, we have
   \begin{align*}
       V_{M_k}^{\pi_k}(\bx) \ge V_{M}^{\pi_*}(\bx)
    \end{align*}
\end{lemma}
\begin{proof}
Recall that at episode $k$, we define the optimistic MDP of MB-UCBVI by $M_k$ in which the parameters of any arm $a\in[n]$ are $(\hat{\br}_{k-1}^a +b_{k-1}^a, \Qhat_{k-1}^a)$ with $b_{k-1}(x_a){=}\frac{L_{k-1}}{2(1-\beta)\sqrt{\max\{1,N_{k-1}(x_a)\} }}$ for any $x_a\in\calS^a$.
The Gittins index policy $\pi_k$ is optimal for MDP $M_k$.
For any state $\bx$, let $a=\pi_k(\bx)$ and $a_*=\pi_*(\bx)$.
Then,
\begin{align*}
    V_{M_k}^{\pi_k}(\bx) -V_M^{\pi_*}(\bx)
    &= b_{k-1}(x_a) +\rhat_{k-1}(x_a) +\beta\sum_{\by}\Phat_{k-1}^{a}(\bx,\by)V_{M_k}^{\pi_k}(\by) -V_M^{\pi_*}(\bx)\\ 
    &\ge b_{k-1}(x_{a_*}) +\rhat_{k-1}(x_{a_*}) +\beta\sum_{\by}\Phat_{k-1}^{a_*}(\bx,\by)V_{M_k}^{\pi_k}(\by) \\
    &\quad -r(x_{a_*}) -\beta\sum_{\by}P^{a_*}(\bx,\by)V_M^{\pi_*}(\by) \\
    &= b_{k-1}(x_{a_*}) {+}\rhat_{k-1}(x_{a_*}) {-}r(x_{a_*}) {+}\beta\sum_{\by}(\Phat_{k-1}^{a_*}(\bx,\by) {-}P^{a_*}(\bx,\by))V_M^{\pi_*}(\by) \\
    &\qquad +\beta \sum_{\by}\Phat_{k-1}^{a_*}(\bx,\by)(V_{M_k}^{\pi_k}(\by) -V_M^{\pi_*}(\by))
\end{align*}
In matrix form, we have
\begin{align*}
    V_{M_k}^{\pi_k} -V_M^{\pi_*}
    &\ge b_{k-1}^{\pi_*} +\rhat^{\pi_*}_{k-1} -r^{\pi_*} {+}\beta(\Phat_{k-1}^{\pi_*} {-}P^{\pi_*})V_M^{\pi_*} +\beta \Phat_{k-1}^{\pi_*}(V_{M_k}^{\pi_k} -V_M^{\pi_*})
\end{align*}
Under event $\calE_{k-1}^\text{UCBVI}$, $b_{k-1}^{\pi_*} +\rhat^{\pi_*}_{k-1} -r^{\pi_*} {+}\beta(\Phat_{k-1}^{\pi_*} {-}P^{\pi_*})V_M^{\pi_*}\ge 0$. This implies that:
\begin{align*}
    (I-\beta \Phat_{k-1}^{\pi_*})(V_{M_k}^{\pi_k} -V_M^{\pi_*}) &\ge 0.
\end{align*}
As $(I-\beta \Phat_{k-1}^{\pi_*})^{-1} = I+ (I-\beta \Phat_{k-1}^{\pi_*}) +(I-\beta \Phat_{k-1}^{\pi_*})^2 +\dots$ is a matrix whose coefficients are all non-negative, this implies that $V_{M_k}^{\pi_k} -V_M^{\pi_*} \ge 0$. 
\end{proof}

\subsubsection{Analysis of $\esp{\Delta^{conc}_k\ind{\calE^\text{UCBVI}_{k-1}}}$ for MB-UCBVI}
Following \eqref{eq:gap_overview}, the expected regret can be written similarly to Equation~\eqref{eq:proof_UCRL2_regret} for MB-UCRL2, one can write that
\begin{align*}
    \esp{\Reg(K,\text{MB-UCBVI},M)}
    &\le \frac{7}{1-\beta} {+} \sum_{k=1}^{K} \esp{\Delta^{model}_k\ind{\calE^\text{UCBVI}_{k-1}}} {+} \esp{\Delta^{conc}_k\ind{\calE^\text{UCBVI}_{k-1}}}.
\end{align*}
The same as MB-UCRL2, the second term is non-positive.
We are therefore left with the last term.
Using Lemma~\ref{lem:regret_decomposition} with $B_k=\frac{H_kL_{k-1}}{2(1-\beta)}$ and the definition of $M_k$ for MB-UCBVI, we have:
\begin{align*}
    \sum_{k=1}^K \esp{ \ind{\calE_{k{-}1}^\text{UCBVI}} \Delta_k^\text{conc}} &{\le} \sum_{k=1}^K \mathbb{E}\bigg[ \ind{\calE_{k{-}1}^\text{UCBVI}}\sum_{t=t_k}^{t_{k+1}-1} \abs{b_{k-1}(X_{t,A_t}) {+}\rhat_{k-1}(X_{t,A_t}){-}r(X_{t,A_t})} \\
    &\qquad {+}\frac{H_kL_{k-1}}{2(1-\beta)}\norm{\Qhat_{k-1}(X_{t,A_t},\cdot){-}Q(X_{t,A_t},\cdot)}_1 \bigg] \\
    &\le \esp{\sum_{k=1}^{K}\sum_{t=t_k}^{t_{k+1}-1} \frac{(2-\beta)L_{k-1}+H_kL_{k-1}(L_{k-1}+1.5\sqrt{S})}{2(1-\beta)\sqrt{\max\{1,N_{k-1}(X_{t,A_t})\} }}} \\
    &\le \esp{\frac{2L_K(L_K+\sqrt{S})\max_{k\le K}H_k}{1-\beta} \sum_{k=1}^{K} \sum_{t=t_k}^{t_{k+1}-1} \frac1{\sqrt{\max\{1,N_{k-1}(X_{t,A_t})\} }}} \\
    &\le \esp{\frac{2L_K(L_K+\sqrt{S})\max_{k\le K}H_k}{1-\beta} \p{Sn\max_{k\le K}H_k +2\sqrt{SnK\max_{k\le K}H_k} }} 
\end{align*}
where the second inequality holds due to the definition of $\calE_{k-1}^\text{UCBVI}$ and the last one holds due to Lemma~\ref{lem:sum}.
With $L_K{=}\sqrt{2\log\frac{4SnK^2\log K}{1-\beta}}$, we have
\begin{align*}
    \sum_{k=1}^K \esp{ \ind{\calE_{k{-}1}^\text{UCBVI}} \Delta_k^\text{conc}}
    &\le \frac{2(1{+}\sqrt{S})}{1-\beta} 2\log\p{\frac{4SnK^2\log K}{1-\beta}} \p{Sn\esp{\max_{k\le K}H_k^2} {+}2\sqrt{SnK} \esp{\max_{k\le K}H_k^{3/2}}}
\end{align*}
The last term of the right side above can be analyzed exactly the same as what is done for \eqref{eq:proof_bound} using Lemma~\ref{lem:moment}.
This concludes the proof.

%

\section{Proof of Theorem~\ref{thm:lower_bound}}
\label{apx:sketch_of_proof_lower}

To prove the lower bound, we consider a specific Markovian bandit problem that is composed of $S$ independent {\it stochastic bandit problems}. 
This allows us to reuse the existing minimax lower bound for stochastic bandit problems. 
This existing result can be stated as follows: let $\Algo\sto$ be a learning algorithm for the stochastic bandit problem. 
It is shown in Theorem~3.1 of \cite{bubeck2012regret} that for any number of arms $n$ and any number of time steps $\tau$, there exists parameters for a stochastic bandit problem $M\sto$ with $n$ arms such that the regret of the learning algorithm over $\tau$ time steps is at least $(1/20)\sqrt{n \tau}$. 
\begin{align}
    \label{eq:regret_sto}
    \Reg\sto( \tau, \Algo\sto, M\sto) \ge \frac1{20}\sqrt{n\tau}. 
\end{align}
This lower bound (Theorem~3.1 of \cite{bubeck2012regret}) is constructed by considering $n$ stochastic bandit problems $M\stoi$ for $j\in[n]$ with parameters that depend on $\tau$ and $n$. 
In the problem $M\stoi$, all arms have a reward $\gamma(\tau,n)$ except arm $j$ that has a reward $\gamma'(\tau,n)>\gamma(\tau,n)$. 
It is shown in Theorem~3.1 of \cite{bubeck2012regret} that a learning algorithm cannot perform uniformly well on all problems because it is impossible to distinguish them {\it a priori}. 
More precisely, in the proof of Lemma~3.2 of \cite{bubeck2012regret}, it is shown that if the best arm is chosen at random, then the expected (Bayesian) regret of any learning algorithm is at least $(1/20)\sqrt{n \tau}$.

As for our problem, let $K$ be a number of episodes, $\beta$ a discount factor, $n$ a number of arms, $S$ a number of states per arm and set $\tau=K/(2S(1-\beta))$. 
We consider a random Markovian bandit model $M$ constructed as follows. 
Each arm $a$ has $S$ states with the state space $\calS^a = \{ 1_a,2_a,\ldots, S_a\}$. 
The transition matrix $Q_a$ is the identity matrix. 
For each state $i\in\{1\dots S\}$, we choose the best arm $a^*_i$ uniformly at random among the $n$ arms, independently for each $i$. 
The rewards of a state $i_a$ are \emph{i.i.d.} Bernoulli rewards with mean $\gamma(\tau,n)$ if $a\ne a^*_i$ and $\gamma'(\tau,n)$ if $a=a^*_i$. 
The initial distribution $\rho$ couples the initial states of all arms for all $i\in\{1\dots S\}$, 
\begin{align*}
    \Proba{\forall a\in[n]: x_{0,a} = i_a} = \frac1S. 
\end{align*} 
In this case, the Markovian bandit problem becomes a combination of $S$ independent stochastic bandit problems with $n$ arms each. 
We denote by $M\sto_i$ the random stochastic bandit problem for the initial state $\bi=(i_a)_{a\in[n]}$. As the $a^*_i$ are chosen independently, a learning algorithm $\Algo$ cannot use the information for $M\sto_i$ to perform better on $M\sto_j$, $j\ne i$.

Let $\phi$ be the distribution of the random Markovian bandit model $M$ defined above and let $T_i$ be the number of time steps spent in state $\bi$ by the learning algorithm $\Algo$.
\begin{align}
    \BayReg(K,\Algo, \phi)
    &\ge \sum_{i=1}^S\esp{\Reg\sto(T_i, \Algo\sto_i, M\sto_i)} \nonumber\\
    &\ge \sum_{i=1}^S\esp{\Reg\sto(\tau, \Algo\sto_i,  M\sto_i)\ind{T_i\ge\tau}}\label{eq:non-decreasing}\\
    &\ge \frac{S}{20} \sqrt{n\tau}\Proba{T_i\ge \tau}\label{eq:lower2}\\
    &= \frac1{20\sqrt{2}}\Proba{T_i\ge \tau} \sqrt{\frac{SnK}{1-\beta}},\label{eq:tau}
\end{align}
where \eqref{eq:non-decreasing} is true because the expected regret is non-decreasing function of the number of episodes, \eqref{eq:lower2} comes from \eqref{eq:regret_sto} and \eqref{eq:tau} from the definition of $\tau$.

We show in the Lemma~\ref{lem:concentration_T_i} below that $\Proba{T_i\le K/(2S(1-\beta))}\le 8S/K$. This shows that for $K\ge16S$, one has $\Proba{T_i\ge \tau}\ge1/2$. This concludes the proof as $40\sqrt{2}\le60$. 

\begin{lemma}
    \label{lem:concentration_T_i}
    Recall that $T_i$ is the number of time steps that the MDP is in state $\bi$ for the MDP model above. Let $G_k$ be a sequence of \emph{i.i.d.} Bernoulli random variable of mean $1/S$ and let $H_k$ be an independent \emph{i.i.d.} sequence of geometric random variable of parameter $1-\beta$. Then:
    \begin{itemize}
        \item[(i)] $T_i\sim\sum_{k=1}^K G_k H_k$,
        \item[(ii)] $\esp{T_i} = K/(S(1-\beta))$,
        \item[(iii)] $\Proba{T_i \ge \esp{T_i}/2} \ge 1-8S/K$. 
    \end{itemize}
\end{lemma}
\begin{proof}
    Let $G_k$ be a random variable that equals $1$ if the initial state $\bi$ is chosen at the beginning of episode $k$ and recall that $H_k$ is the episode length. By definition, the variables $G_k$ and $H_k$ are independent and follow respectively Bernoulli and geometric distribution. This shows \emph{(i)}. 
    
    Let $W_k=G_kH_k$. As the $W_k$ are \emph{i.i.d.} and $G_k$ and $H_k$ are independent, we have: 
    \begin{align*}
        \esp{T_i} = K\esp{H_1G_1} &= \frac{K}{S(1-\beta)}.
    \end{align*}
    This shows (ii). 
    
    Moreover, $\var{T_i} = K\var{H_1G_1}$. Hence, by using Chebyshev's inequality, one has:
    \begin{align*}
        \Proba{T_i \le \frac{\esp{T_i}}{2}} &\le \Proba{\abs{T_i-\esp{T_i}} \ge \frac{\esp{T_i}}{2}}\\
        &\le \frac{4\var{T_i}}{(\esp{T_i})^2}\\
        &= \frac{4}{K} \frac{\var{H_1G_1}}{(\esp{H_1G_1})^2}.
    \end{align*}

    Concerning the variance, the second moment of a geometric random variable of parameter $1-\beta$ is $(1+\beta)/(1-\beta)^2$. This shows that $\esp{(H_1G_1)^2}=(1+\beta)/(S(1-\beta)^2)\le 2S(\esp{H_1G_1})^2$. This implies:
    \begin{align*}
        \var{H_1G_1} &\le (2S-1)(\esp{H_1G_1})^2\le 2S(\esp{H_1G_1})^2.        
    \end{align*}
    This implies (iii).
\end{proof}

\section{Proof of Theorem~\ref{thm:no_OFU}}
\label{apx:proof_OFU}

\tikzstyle{state}=[circle, draw]
\tikzstyle{reward}=[node distance=0.6cm,font=\small]

\begin{figure}[ht]
    \centering
    \begin{tabular}{c|c|c}
        \begin{tikzpicture}[xscale=0.7]
            \node[state] at (0,0) (A) {$A_1$};
            \node[state] at (2,0) (B) {$A_2$};
            \node[state] at (4,0) (C) {$A_3$};
            \node[below of=A,reward] {$+3$};
            \node[below of=B,reward] {$+4$};
            \node[below of=C,reward] {$+0$};
            \draw (A) edge[loop above, ->] node[above]{$0.5$} (A);
            \draw (A) edge[->] node[above]{$0.5$} (B);
            \draw (B) edge[->] node[above]{$1$} (C);
            \draw (C) edge[loop above, ->] node[above]{$1$} (C);
        \end{tikzpicture}
        &\begin{tikzpicture}[xscale=0.7]
            \node[state] at (0,0) (A) {$B_1$};
            \node[state] at (2,0) (B) {$B_2$};
            \node[state] at (4,0) (C) {$B_3$};
            \node[below of=A,reward] {$+3.21$};
            \node[below of=B,reward] {$+0$};
            \node[below of=C,reward] {$+3.21$};
            \draw (A) edge[->] node[above]{$1$} (B);
            \draw (B) edge[loop above, ->] node[above]{$1$} (B);
            \draw (C) edge[loop above, ->] node[above]{$1$} (C);
        \end{tikzpicture}
        &\begin{tikzpicture}[xscale=0.7]
            \node[state] at (0,0) (A) {$C_1$};
            \node[below of=A,reward] {$+\mu$};
            \draw (A) edge[loop above, ->] node[above]{$1$} (A);
        \end{tikzpicture}
        \\
        (a) $\hat{Q}_a$ and $\hat{r}_a=r_a$.
        &(b) $\hat{Q}_b=Q_b$ and $\hat{r}_b=r_b$.
        &(c) $\hat{Q}_c=Q_c$ and $\hat{r}_c=r_c$.
    \end{tabular}
        \caption{Counterexample for OFU indices: $\hatB_a$, $\hatB_b=\calB_b$, $\hatB_c=\calB_c$.}
        \label{fig:counter-example1}
\end{figure}

In this proof, we reason by contradiction and assume that there exists a procedure that computes local indices such that the obtained policy is such that for any estimate $\hatB$ and any initial condition $\rho$, then if $M\in\M(\hatB)$, one has 
\begin{align}
    \sup_{M\in\M(\hatB)}V^{\pi^{I^{\hatB}}}_{M}(\rho) \ge \sup_{\pi} V^\pi_{M}(\rho).
\end{align}
In the remaining of this section, we set the discount factor to $\beta=0.5$. For a given state $x_a$, we denote by $I(x_a)$ the local index of state $x_a$ computed by this hypothetically optimal algorithm.

We first consider a Markovian bandit problem with two arms $\{b,c\}$. We consider that these two arms are perfectly estimated (\ie, $\epsilon^r_b(x_b)=\epsilon^Q_b(x_b)=\epsilon^r_c(x_c)=\epsilon^Q_c(x_c)=0$). The Markov chains for these arms are depicted in Figure~\ref{fig:counter-example1}. Their transitions matrices and rewards are
\begin{align*}
    Q_b = \left[\begin{array}{ccc}
        0 & 1 & 0\\
        0 & 1 & 0\\
        0 & 0 & 1
    \end{array}\right]
    \text{ and } r_b = [3.21,0,3.21];
    &&Q_c = [1] \text{ and } r_c = [\mu].
\end{align*}
As the Markovian bandit are perfectly known, the indices $I(B_1)$, $I(B_2)$, $I(B_3)$ and $I(C_1)$ must be such that the obtained priority policy is optimal for the true MDP, that is: states $B_1$ and $B_3$ should have priority over $C_1$ (\ie, $I(B_1)>I(C_1)$ and $I(B_3)>I(C_1)$) if and only if $\mu<3.21$, and state $B_2$ should have priority over $C_1$ (\ie, $I(B_2)>I(C_1)$) if and only if $\mu<0$. This implies that the local indices defined by our hypothetically optimal algorithm must satisfy
\begin{align*}
    I(B_1) = I(B_3) > I(B_2).
\end{align*}
Now, we consider Markovian bandit problems with two arms $\{a,b\}$, where Arm~$b$ is as before. For Arm~$a$, we consider a confidence set $\hatB_a=(\hat{Q}_a,\hat{r}_a,\epsilon^r_a,\epsilon^Q_a)$ where $(\hat{Q}_a,\hat{r}_a)$ are depicted in Figure~\ref{fig:counter-example1}(a) and where $\epsilon^r_a(x_a)=0$ and $\epsilon^Q_a(x_a)=0.2$:
\begin{align*}
    \hat{Q}_a = \left[\begin{array}{ccc}
        0.5 & 0.5 & 0\\
        0 & 0 & 1\\
        0 & 0 & 1
    \end{array}\right]
    \text{ and }\hat{\br}_a = \br_a = [3, 4, 0]
    &&\epsilon_a^Q = [0.1, 0.1, 0.1] \text{ and }\epsilon_a^r=[0,0,0].
\end{align*}
We consider two possible instances of the ``true'' Markovian bandit problem, denoted $M^1$ and $M^2$. For $M^1$, the transition matrix and reward function of the first arm are depicted in Figure~\ref{fig:counter-example2}(a). For $M^2$, they are depicted in Figure~\ref{fig:counter-example2}(b). In both cases, $(Q_b,r_b)$ are as in Figure~\ref{fig:counter-example1}(b). It should be clear that $M^1\in\M$ and $M^2\in\M$. 

\begin{figure}[ht]
    \centering
    \begin{tabular}{c|c}
        \begin{tikzpicture}[xscale=0.7]
            \node[state] at (0,0) (A) {$A_1$};
            \node[state] at (2,0) (B) {$A_2$};
            \node[state] at (4,0) (C) {$A_3$};
            \node[below of=A,reward] {$+3$};
            \node[below of=B,reward] {$+4$};
            \node[below of=C,reward] {$+0$};
            \draw (A) edge[loop above, ->] node[above]{$0.4$} (A);
            \draw (A) edge[->] node[above]{$0.6$} (B);
            \draw (B) edge[->] node[above]{$1$} (C);
            \draw (C) edge[loop above, ->] node[above]{$1$} (C);
        \end{tikzpicture}
        &\begin{tikzpicture}[xscale=0.7]
            \node[state] at (0,0) (A) {$A_1$};
            \node[state] at (2,0) (B) {$A_2$};
            \node[state] at (4,0) (C) {$A_3$};
            \node[below of=A,reward] {$+3$};
            \node[below of=B,reward] {$+4$};
            \node[below of=C,reward] {$+0$};
            \draw (A) edge[loop above, ->] node[above]{$0.6$} (A);
            \draw (A) edge[->,bend right] node[above]{$0.4$} (B);
            \draw (B) edge[->,bend right] node[above]{$0.1$} (A);
            \draw (B) edge[->] node[above]{$0.9$} (C);
            \draw (C) edge[loop above, ->] node[above]{$0.9$} (C);
            \draw (C) edge[->, bend right] node[above]{$0.1$} (A);
        \end{tikzpicture}
        \\
        (a) $(Q^a,r^a)$ for $M^1$ &  (b) $(Q^a,r^a)$ for $M^2$
    \end{tabular}
    \caption{The two instances of $\calB_a^1$ and $\calB_a^2$}
    \label{fig:counter-example2}
\end{figure}

If there exist indices that can be computed locally, then the indices for an arm should not depend on the confidence that one has on the other arms. The indices $I(A_1)$, $I(A_2)$ and $I(A_3)$ must satisfy the following facts: 
\begin{itemize}
    \item $I(A_3)\in(I(B_2),I(B_3))$ because for all Markovian bandit $M\in\M$, state $A_3$ should have priority over state $B_2$ and should not have priority over state $B_3$ (because of the discount factor $\beta=1/2$).
    \item $I(A_2)>I(B_1)=I(B_3)$ because for all Markovian bandit $M\in\M$, state $A_2$ will give a higher instantaneous reward than state $B_1$ or $B_3$. It should therefore have a higher priority.
\end{itemize}
This leaves two possibilities for $I(A_1)$: 
\begin{itemize}
\item If $I(A_1)>I(B_1)=I(B_3)$, then state $A_1$ has priority over both $B_1$ and $B_3$.  We denote the corresponding priority policy $\pi^1$.
\item If $I(A_1)<I(B_1)=I(B_3)$, then state $B_1$ and $B_3$ have a higher priority than state $A_1$. We denote the corresponding priority policy by $\pi^2$.
\end{itemize}  

We use a numerical implementation of extended value iteration (available in the supplementary material) to find that:
\begin{align}
    \sup_{M\in\mathbb{M}} V^{\pi^2}_{M}(A_1,B_3) \approx 6.42
    &< \sup_{\pi}V^{\pi}_{M^1}(A_1,B_3) \approx 6.47 \label{eq:counter-example1}\\
    \sup_{M\in\mathbb{M}} V^{\pi^1}_{M}(A_1,B_1) \approx 5.96
    &< \sup_{\pi}V^{\pi}_{M^2}(A_1,B_1) \approx 6.00 \nonumber
\end{align}
This implies that there does not exist any definition of indices such that \eqref{eq:optimism} holds regardless of $M$ and $\bx$. 

\section{Description of the Algorithms and Choice of Hyperparameter}
\label{apx:algos}

In this section, we provide a detailed description of the simulation environment used in the paper. We first describe the Markov chain used in our example. Then, we describe all algorithms that we compare in the paper. For each algorithm, we give some details about our choice of hyperparameters. Last, we also describe the experimental methodology that we used in our simulations. 

\subsection{Description of the example}
\label{apx:scene1}

We design an environment with 3 arms, all following a Markov chain represented in Table~\ref{fig:randomwalk}.  This Markov chain is obtained by applying the optimal policy on the river swim MDP of \cite{filippiOptimismReinforcementLearning2010a}. 
In each chain, there are 2 rewarding states: state 1 with low mean reward $r_L$, and state 4) with high mean reward $r_R$, both with Bernoulli distributions. At the beginning of each episode, all chains start in their state 1. Each chain is parametrized by the values of $p_L,p_R,p_{RL},r_L,r_R$ that are given in Table~\ref{fig:randomwalk} along
with the corresponding Gittins indices of each chain.

\begin{table}[htbp]
    \begin{tabular}{@{}c@{}c@{}}
        \begin{tabular}{@{}c@{}}
            \includegraphics[angle=0, width=0.43\linewidth]{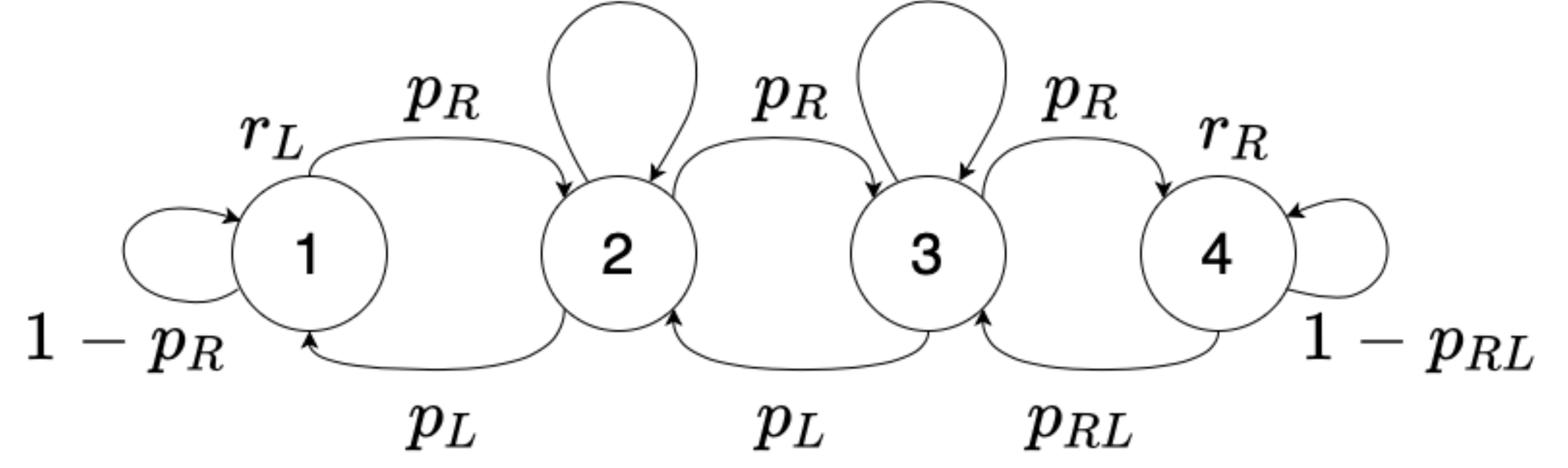}    
        \end{tabular}
        &\begin{tabular}{|@{~}c@{~}|@{~}c@{~}|@{~}c@{~}|@{~}c@{~}|@{~}c@{~}|c@{~}|@{~}c@{~}|@{~}c@{~}|@{~}c@{~}|}
            \cline{6-9}
            \multicolumn{5}{@{}c@{}}{} 
            & \multicolumn{4}{@{}|@{~}c@{~}| }{Gittins index for each state} \\ \cline{1-9}
             $p_L$ & $p_R$ & $p_{RL}$ & $r_L$ & $r_R$ & 1 & 2 & 3 & 4\\ \cline{1-9}
              0.1  &  0.2  &   0.3    &  0.2  &  1.0  & 0.276 & 0.2894 & 0.392 & 1.0 \\ \cline{1-9}
              0.1  &  0.5  &   0.7    & 0.35  &  0.7  & 0.35 & 0.256 & 0.2892 & 0.7 \\ \cline{1-9}
              0.1  &  0.4  &   0.5    &  0.4  &  0.65 & 0.4 & 0.250 & 0.286 & 0.65 \\ \cline{1-9}
            \end{tabular}            
    \end{tabular}
    \vspace{0.2cm}
    \caption{The random walk chain with 4 states. In state 4, the chain has an average reward $r_R$. For state 2 and 3, the chain gives zero reward. In state 1, the mean reward is $r_L$. This chain is obtained by applying the optimal policy on the 4-state river swim MDP of \cite{filippiOptimismReinforcementLearning2010a}. The table contains the parameters that we used, along with Gittins indices of all states when the discount factor is $\beta=0.99$.}
    \label{fig:randomwalk}
\end{table}

\subsection{\bf MB-PSRL}
\label{sssec:impl_gittinsPS}

MB-PSRL, the adaption from PSRL, puts prior distribution on the parameters $(\br^a,Q^a)$ of each Arm$~a$, draws a sample from the posterior distribution and uses it to compute the Gittins indices at the start of each episode. We implement two posterior updates for the mean reward vector $\br^a$: Beta and Gaussian-Gamma. The second posterior, Gaussian-Gamma, will be used in prior choice sensitivity tests. For the transition matrix $Q^a$, we implemented Dirichlet posterior update because Dirichlet distribution is the only natural conjugate prior for categorical distribution. Beta, Gaussian-Gamma and Dirichlet distributions can be easily sampled using the numpy package of Python. This greatly contributes to the computational efficiency of MB-PSRL.

We give more details on this prior distribution and their conjugate posterior in the subsections below.

\subsubsection{Bayesian Updates: Conjugate Prior and Posterior Distributions}
\label{apx:Bayes}

MB-PSRL is a Bayesian learning algorithm. 
As such, it samples reward vectors and transition matrices at the start each episode.
We would like to emphasize that neither the definition of the algorithm nor its performance guarantees that we prove in Theorem~\ref{thm:regret_upper_bound} depend on a specific form of the prior distribution $\phi$. 
Yet, in practice, some prior distributions are more preferable because their conjugate distributions are easy to implement. 
In the following, we give concrete examples on how to update the conjugate distribution given the observations. 

For $a\in[n]$ and $x_a\in\calS^a$, let $N_{k-1}(x_a)$ be the number of activations of arm $a$ while in state $x_a$ up to episode $k$. 
For this state $x_a$, the number of samples of the reward and of transitions from $x_a$ are equal to $N_{k-1}(x_a)$. 
To ease the exposition, we drop the label $a$ and assume that we are given: 
\begin{itemize}
    \item $N_{k-1}(x)$ \emph{i.i.d.} samples $\{Y_1,\dots,Y_{N_{k-1}(x)}\}$ of next states to which the arm transitioned from $x$.
    \item $N_{k-1}(x)$ \emph{i.i.d.} samples $\{R_1,\dots,R_{N_{k-1}(x)}\}$ of random immediate rewards earned while the arm was activated in state $x$
\end{itemize}
Each $Y_i$ is such that $\Proba{Y_i=y}=Q(x,y)$ and each $R_i$ is such that $\esp{R_i}=r(x)$. In what follows, we describe natural priors that can be used to estimate the transition matrix and the reward vector. 

\subsubsection{Transition Matrix}
If no information is known about the arm, the natural prior distribution is to consider the lines $Q(x,\cdot)$ of the matrix as independent multivariate random variables uniformly distributed among all non-negative vectors of length $S$ that sum to $1$. 
This corresponds to a Dirichlet distribution of parameters $\alpha=(1,\dots,1)$. 
For a given $x$, the variables $\{Y_1,\dots,Y_{N_{k-1}(x)}\}$ are generated according to a categorical distribution $Q(x,\cdot)$. 
The Dirichlet distribution is self-conjugate with respect to the likelihood of a categorical distribution. 
So, the posterior distribution $\phi(Q(x,\cdot)| Y_1,\dots,Y_{N_{k-1}(x)})$ is a Dirichlet distribution with parameters $\mathbf{c}=(c_1\dots c_S)$ where $c_y=1+\sum_{i=1}^{N_{k-1}(x)}\ind{Y_i=y}$.

\subsubsection{Reward Distribution}
\label{apx:reward_post}

As for the reward vector, the choice of a good prior depends on the distribution of rewards. 
We consider two classical examples: Bernoulli and Gaussian.

\paragraph{Bernoulli distribution}
A classical case is to assume that the reward distribution of a state $x$ is Bernoulli with mean value $r(x)$.
A classical prior in this case is to consider that $\{r(x)\}_{\{x\in\calS\}}$ are \emph{i.i.d.} random variables following a uniform distribution whose support is $[0,1]$. 
The posterior distribution of $r(x)$ at time $t$ is the distribution of $r(x)$ conditional to the reward observations from state $x$ gathered up to time $t$. 
The posterior distribution $\phi(r(x)\mid R_1,\dots,R_{N_{k-1}(x)})$ is then a Beta distribution with parameters $(1+\sum_{i=1}^{N_{k-1}(x)}\ind{R_i=1}, 1+\sum_{i=1}^{N_{k-1}(x)}\ind{R_i=0})$. 
Recall that the Beta distribution is a special case of the Dirichlet distribution in the same way as the Bernoulli distribution is a special case of the Categorical distribution. 

\paragraph{Gaussian distribution}
We now consider the case of Gaussian rewards and we assume that the immediate rewards earned in state $x$ are \emph{i.i.d.} Gaussian random variables of mean and variance $(r(x), \sigma^2(x))$. 
A natural prior for Gaussian rewards is to consider that $\{{(r(x), \frac1{\sigma^2(x)})}\}_{\{x\in\calS\}}$ are \emph{i.i.d.} bivariate random variables where the marginal distribution of each $\frac{1}{\sigma^2(x)}$ is a Gamma distribution (it is a natural belief since the empirical variance of Gaussian has a chi-square distribution which is a special case of Gamma distribution). 
Conditioned on $\frac1{\sigma^2(x)}$, $r(x)$ follows a Gaussian distribution of variance $\sigma^2(x)$. 
We say that $(r(x), \frac1{\sigma^2(x)})$ has a Gaussian-Gamma distribution, which is self-conjugate with respect to a Gaussian likelihood (\emph{i.e.,} the likelihood of Gaussian rewards). 
So, given the reward observations, the marginal distribution of $\frac{1}{\sigma^2(x)}$ is still a Gamma distribution. 
$r(x)$ has Gaussian distribution conditioned on the reward observations and $\frac1{\sigma^2(x)}$. 
Indeed, let $\hat{r}(x)=\frac{1}{N_{k-1}(x)}\sum_{i=1}^{N_{k-1}(x)}R_i$ and $\hat{\sigma}^2(x)=\frac{1}{N_{k-1}(x)}\sum_{i=1}^{N_{k-1}(x)}\left(R_i-\hat{r}(x)\right)^2$ be the empirical mean and empirical variance of $R_i$. Then it can be shown that the posterior distribution of $\frac{1}{\sigma^2(x)}$ and $r(x)$ are:
\begin{align*}
    \frac{1}{\sigma^2(x)}\mid R_1,\dots,R_{N_{k-1}(x)}&{\sim} 
    \mathrm{Gamma}\bigg(\frac{N_{k-1}(x){+}1}{2}, \frac{1}{2}{+}\frac{N_{k-1}(x)\hat{\sigma}^2(x)}{2} {+}\frac{N_{k-1}(x)\hat{r}^2(x)}{2(N_{k-1}(x){+}1)}\bigg)\\
    r(x)\mid \frac1{\sigma^2(x)}, R_1,\dots,R_{N_{k-1}(x)}&{\sim} \mathcal{N}\left(\frac{N_{k-1}(x)\hat{r}(x)}{N_{k-1}(x)+1}, \frac{\sigma^2(x)}{N_{k-1}(x)+1}\right).
\end{align*}
For more details about the analysis of conjugate prior and posterior presented above as well as more conjugate distributions, we refer the reader to \cite{finkCompendiumConjugatePriors1997,murphyConjugateBayesianAnalysis2007}.

Notice that a reward that has a Gaussian distribution violates the property that all rewards are in $[0,1]$.
This could invalidate the bound on the regret of our algorithm proven in Theorem \ref{thm:regret_upper_bound}. 
Actually, it is possible to correct the proof to cover the Gaussian case by replacing the Hoeffding's inequality used in Lemma \ref{lem:concentration} by a similar inequality, also valid for sub-Gaussian random variables, see \cite{Vershynin}. 
In the experimental section (see \ref{ssec:prior}), we also show that a bad choice for the prior distribution of the reward (assuming a Gaussian distribution while the rewards are actually Bernoulli) does not alter too much the performance of the learning algorithm.





\subsection{Experimental Methodology}
\label{ssec:experimental_methodo}

In our numerical experiment, we did 3 scenarios to evaluate the algorithms (scenario 2 and 3 are given in Appendix~\ref{apx:add_numerical}). In each scenario, we choose the discount factor $\beta=0.99$ (which is classical) and we compute the regret over $K=3000$ episodes. The number of simulations varies over scenario depending on how the regret is computed. For each run, we draw a sequence of horizons $\{H_k\}_{k\in[3000]}$ from a geometric distribution of parameter $0.01$ and we run all algorithms for this sequence of time-horizons to remove a source of noise in the comparisons.  

For a given sequence of policies $\pi_k$, following Equation~\eqref{eq:regretTraj}, the expected regret is $\esp{\sum_{k=1}^K\Delta_k(\bX_{t_k})}$ where $\Delta_k(\bX_{t_k})$ is the expected regret over episode $k$.
To reduce the variance in the numerical experiment, we compute
$\Delta_k(\bX_{t_k})=V^{\pi_*}_M(\bX_{t_k}) -V^{\pi_k}_M(\bX_{t_k})$.
For a given Markovian bandit problem and state $\bx$, the value $V^{\pi_*}_{M}(\bx)$ can be computed by using the retirement evaluation presented in Page~272 of \cite{whittle1996optimal}. It seems, however, that the same methodology is not applicable to compute the value function of an index policy that is not the Gittins policy. This means that while the policy $\pi_k$ is easily computable, we do not know of an efficient algorithm to compute its value $V^{\pi_k}_{M}(\bx)$. Hence, in our simulations, we will use two methods to compute the regret, depending on the problem size:
\begin{enumerate}
    \item (Exact method) Let $(r^{\pi}, P^{\pi})$ be the reward vector and transition matrix under policy $\pi$ (i.e. $\forall \bx, \by\in\calE, r^{\pi}(\bx)=r(\bx,\pi(\bx)), P^{\pi}(\bx,\by)=P^{\pi(\bx)}(\bx, \by)$). Using the Bellman equation, the value function under policy $\pi$ is computed by
    \begin{equation}
        \label{eq:exact_value}
        V_{M}^{\pi}=(\pmb{1}-\beta P^{\pi})^{-1}r^{\pi}.
    \end{equation}
    The matrix inversion can be done efficiently with the numpy package of Python. However, this takes $S^{2n}+2S^{n}$ of memory storage. Hence, when the number of states and arms are too large, the exact computation method cannot be performed. 

    \item (Monte Carlo method) In Scenario~2, the model has $n=9$ arms with $S=11$ states each, which makes the exact method inapplicable. In this case, it is still possible to compute the optimal policy and to apply Gittins index based algorithms but computing their value is intractable. In such a case, to measure the performance, we do 240 simulations for each algorithm and try to approximate $\Delta_k$ by
    \begin{equation}
        \label{eq:simulation_value}
        \hat{\Delta}_k=\frac1{\text{\#replicas}}\sum_{j=1}^{\text{\#replicas}}\sum_{t=0}^{H_k^{(j)}-1}\Big[r(X_{t,A^{*,(j)}_t}^{*,(j)}) -r(X_{t,A^{(j)}_t}^{(j)})\Big],
    \end{equation}
    where $H_k^{(j)}$ is the horizon of the $k$th episode of the $j$th simulation and $\{X_{t,A^{*,(j)}_t}^{*,(j)}\}$ and $\{X_{t,A^{(j)}_t}^{(j)}\}$ are the trajectories of the oracle and the agent respectively. The term oracle refers to the agent that knows the optimal policy.
\end{enumerate}
Note that the expectation of \eqref{eq:simulation_value} is equal to the value given in \eqref{eq:exact_value} but \eqref{eq:simulation_value} has a high variance. Hence, when applicable (Scenario~1~and~3) we use Equation~\eqref{eq:exact_value} to compute the expected regret.

\section{Additional Numerical Experiments}
\label{apx:add_numerical}

\subsection{Scenario 1: Small Dimensional Example (Random Walk chain)}
\label{sssec:random_walk}

This scenario is explained in Appendix~\ref{apx:scene1} and the main numerical results are presented in Section~\ref{sec:numerical}. Here, we provide the result with error bars with respect to the random seed. The error bar size equals twice the standard deviation over 80 samples (each sample is a simulation with a given random seed and the random seeds are different for different simulations).

\begin{figure}[ht]
    \center
    \includegraphics[width=0.9\linewidth]{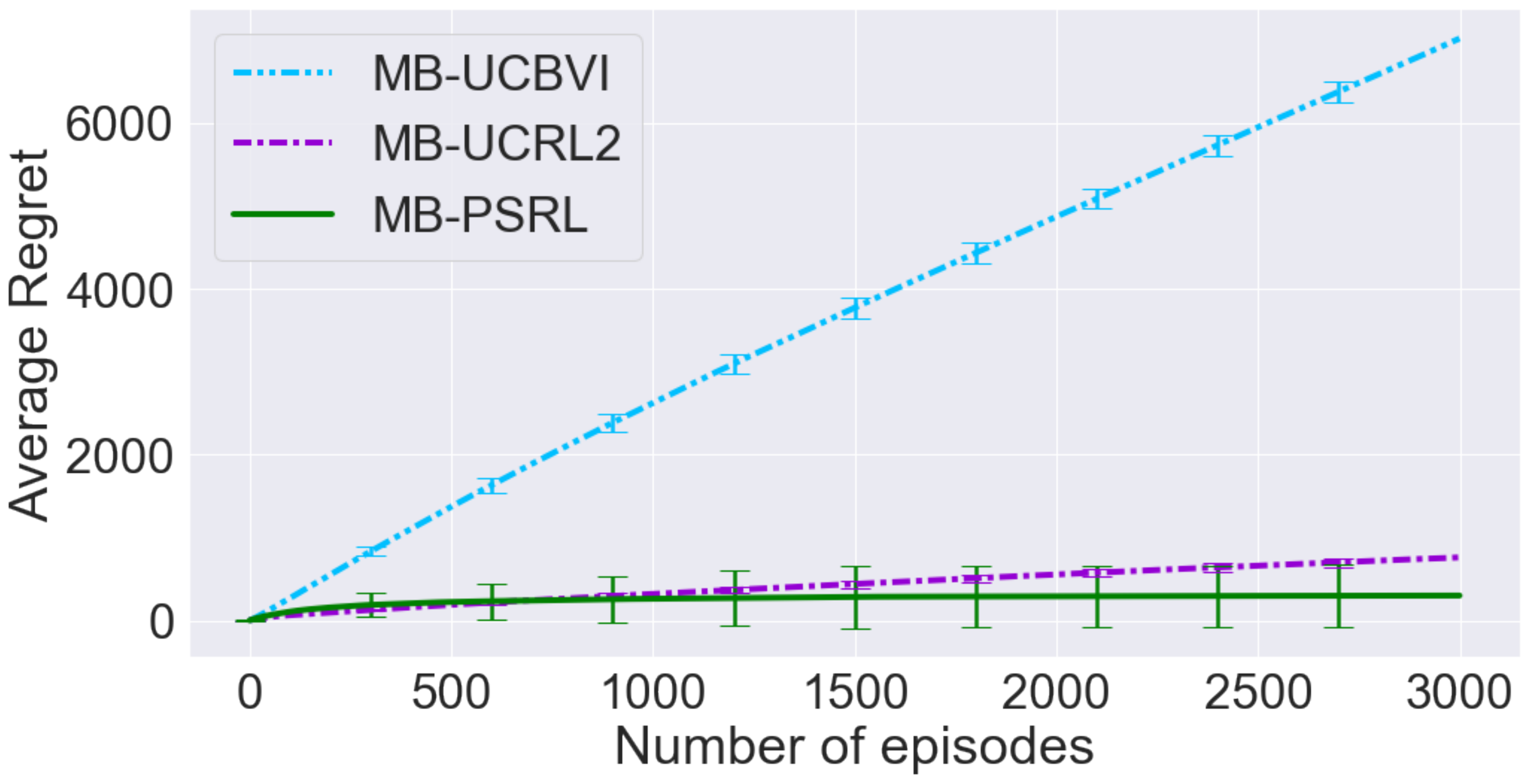}
    \caption{Average cumulative regret in function of the number of episodes. Result from 80 simulations in a Markovian bandit problem with three 4-state random walk chains given in Table~\ref{fig:randomwalk}. The horizontal axis is the number of episodes. The size of the error bar equals twice the standard deviation over 80 simulations.}
    \label{fig:randomwalk_3b4s_errorbar}
\end{figure}

\subsection{Scenario 2: Higher Dimensional Example (Task Scheduling)}
\label{sssec:task_scheduling}

We now study an example that is too large to apply MB-UCRL2 . Hence, here we only compare MB-PSRL and MB-UCBVI. 

We implement the environment proposed on page 19 of \cite{duffQLearningBanditProblems1995a} that was used as a benchmark for the algorithm in the cited paper. Each chain represents a task that needs to be executed, and is represented in \figurename~\ref{fig:task_scheduling}(a). 
Each task has 11 states (including finished state $\star$ that is absorbing). 
For a given chain $a\in\{1,\dots,9\}$ and a state $i\in\{1,\dots,10\}$, the probability that a task $a$ ends at state $i$ is $\rho^{(a)}_{i}=\sProba{\tau^{(a)}=i \mid\tau^{(a)}\ge i}$ where $\tau^{(a)}$ is the execution time of task $a$. 
We choose the same values of the parameters as in \cite{duffQLearningBanditProblems1995a}: $\rho^{(a)}_{1}=0.1a$ for $a\in\{1,\dots,9\}$, $\lambda=0.8$, $\beta=0.99$ and for $i\ge2$, 
\begin{align*}
    \mathbb{P}\{x_a=i\}=\{1-[1-\rho^{(a)}_{1}]\lambda^{i-1}\}[1-\rho^{(a)}_{1}]^{i-1}\lambda^{\frac{(i-1)(i-2)}{2}}.
\end{align*}
Hence, the hazard rate $\rho^{(a)}_{i}$ is increasing with $i$. The reward in this scenario is deterministic: the agent receives 1 if the task is finished (\emph{i.e.,} under the transition from any state $i$ to state $\star$) and 0 otherwise (\emph{i.e.,} any other transitions including the one from state $\star$ to itself).
For MB-PSRL, we use a uniform prior for the expected rewards and consider that the rewards are Bernoulli distributed.

\begin{figure}[ht]
    \center
    \begin{tabular}{cc}
        \begin{minipage}{.5\linewidth}
            \includegraphics[angle=270, width=\textwidth]{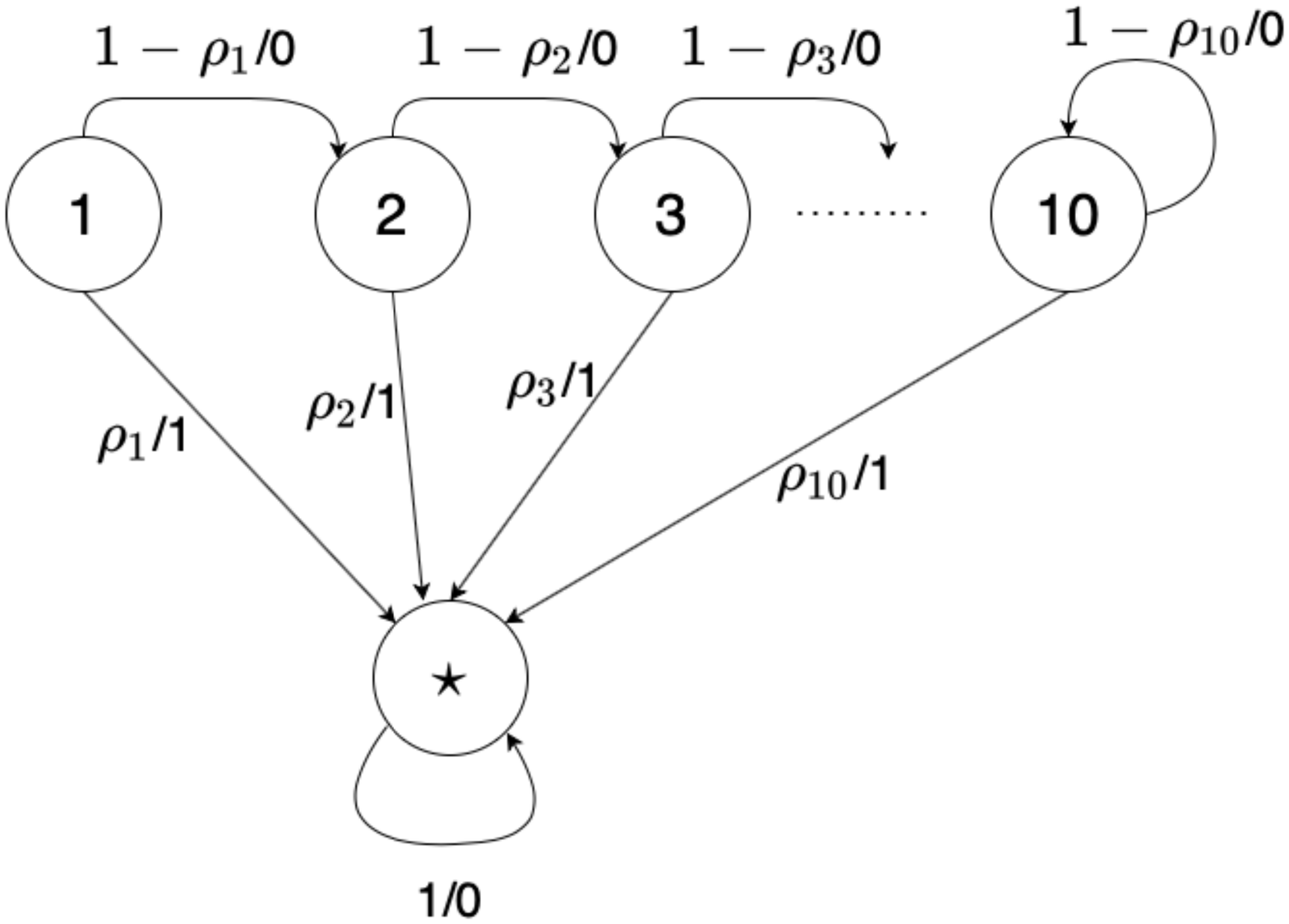}
            (a) In state $i$, the task is finished with probability $\rho_i$ or transitions to state $i+1$ with probability $1-\rho_i$. For $i=1,\dots,10$, the transition from state $i$ to state $\star$ provides 1 as the immediate reward. Otherwise, the agent always receives 0 reward.
        \end{minipage}
        &
        \begin{minipage}{.4\linewidth}
            \includegraphics[angle=0, width=\columnwidth]{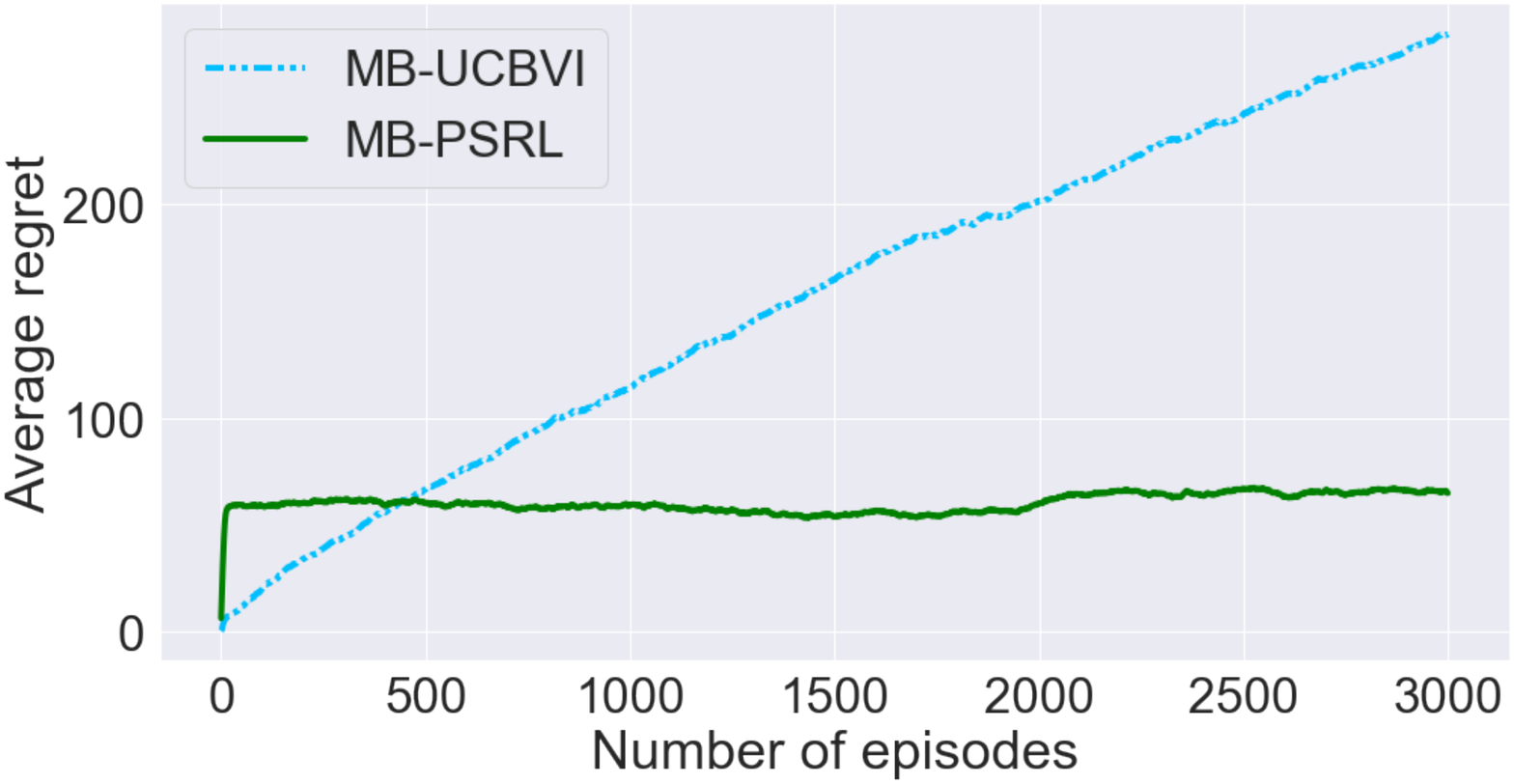}
            (b) Average cumulative regret over 240 simulations.
        \end{minipage}        
    \end{tabular}
    \caption{Task Scheduling with 11 states including the absorbing state (finished state). }
    \label{fig:task_scheduling}
\end{figure}

The average regret of the two algorithms is displayed in \figurename~\ref{fig:task_scheduling}(b). As before, MB-PSRL outperforms MB-UCBVI.  Note that we also studied the time to run one simulation for 3000 episodes. This time is around 1 min for MB-PSRL and MB-UCBVI.

\subsection{Scenario 3: Bayesian Regret and Sensitivity to the Prior}
\label{ssec:prior}

In this section, we study how robust the two implementations of PSRL are, namely MB-PSRL and vanilla PSRL (to simplify, we will just call the later PSRL), to a choice of prior distributions. 
As explained in Appendix~\ref{apx:reward_post}, the natural conjugate prior for Bernoulli reward is the Beta distribution. 
In this section, we simulate MB-PSRL and PSRL in which the rewards are Bernoulli but the conjugate prior used for the rewards are Gaussian-Gamma which is incorrect for Bernoulli random reward. 
In other words, MB-PSRL and PSRL have Gaussian-Gamma prior belief while the real rewards are Bernoulli random variables.

To conduct our experiments, we use a Markovian bandit problem with three $4$-state random walk chains represented in Table~\ref{fig:randomwalk}. 
We draw 16 models by generating 16 pairs of $(r_L, r_R)$ from $U[0,1]$, 16 pairs of $(p_L, p_R)$ from Dirichlet(3,(1,1,1)) and 16 values of $p_{RL}$ from Dirichlet(2, (1,1)) for each chain. 
Each model is an unknown MDP that will be learned by MB-PSRL or PSRL. 
For each of these $16$ models, we simulate MB-PSRL and PSRL 5 times with correct priors and 5 times with incorrect priors. 
The result can be found in \figurename~\ref{fig:bayes_gittinsPS} which suggests that MB-PSRL performs better when the prior is correct and is relatively robust to the choice of priors in term of Bayesian regret. 
This figure also shows that PSRL seems more sensitive to the choice of prior distribution. 
Also note that for both MB-PSRL and PSRL, some trajectories deviate a lot from the mean, under correct priors but even more so with  incorrect priors. 
This illustrates the general fact that learning can go wrong, but with a small probability.

\begin{figure*}[tb]
    \begin{tabular}{cccc}
        \rotatebox{90}{\quad\qquad MB-PSRL}
        &\includegraphics[width=0.3\linewidth]{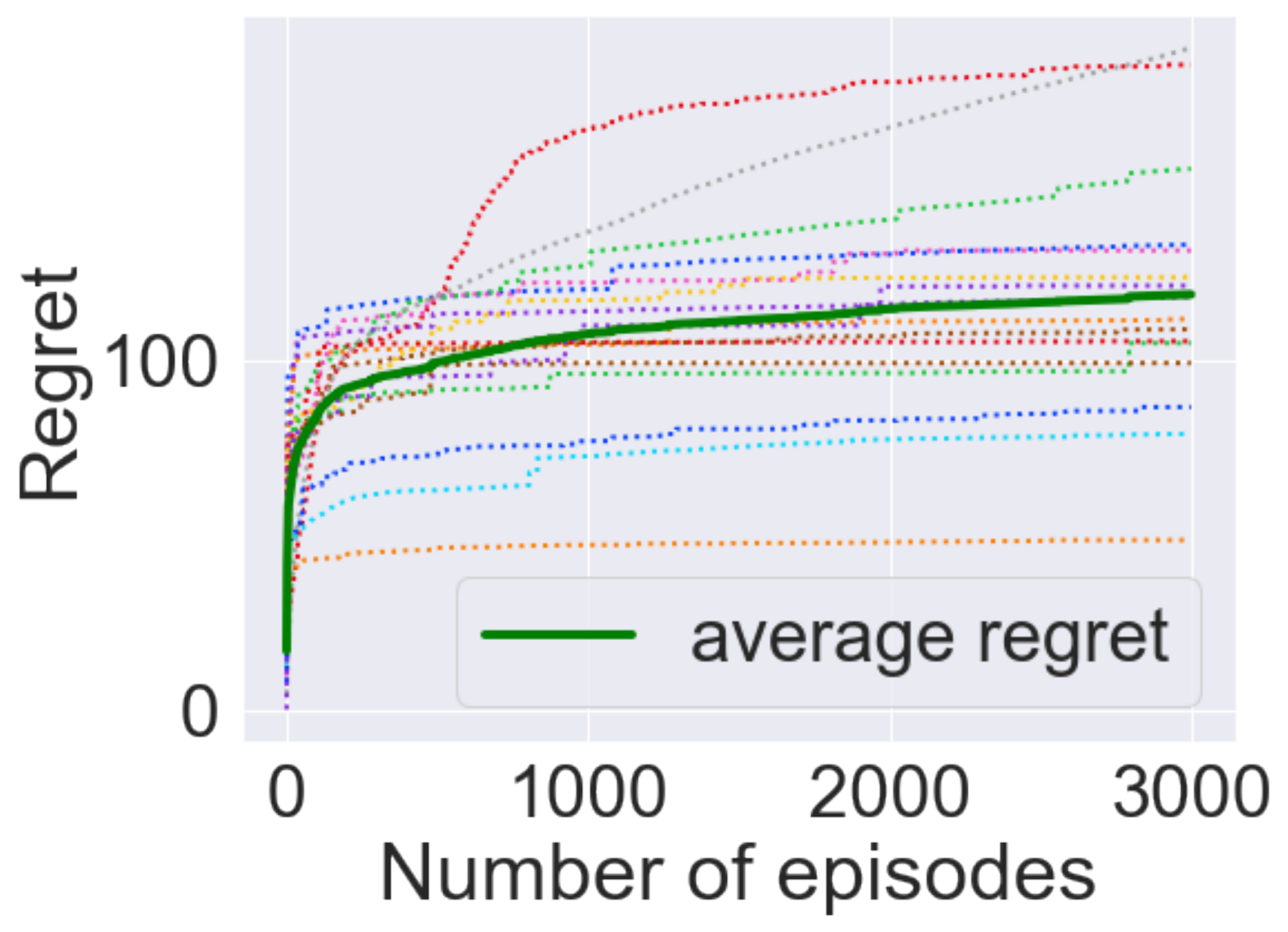}
        &\includegraphics[width=0.3\linewidth]{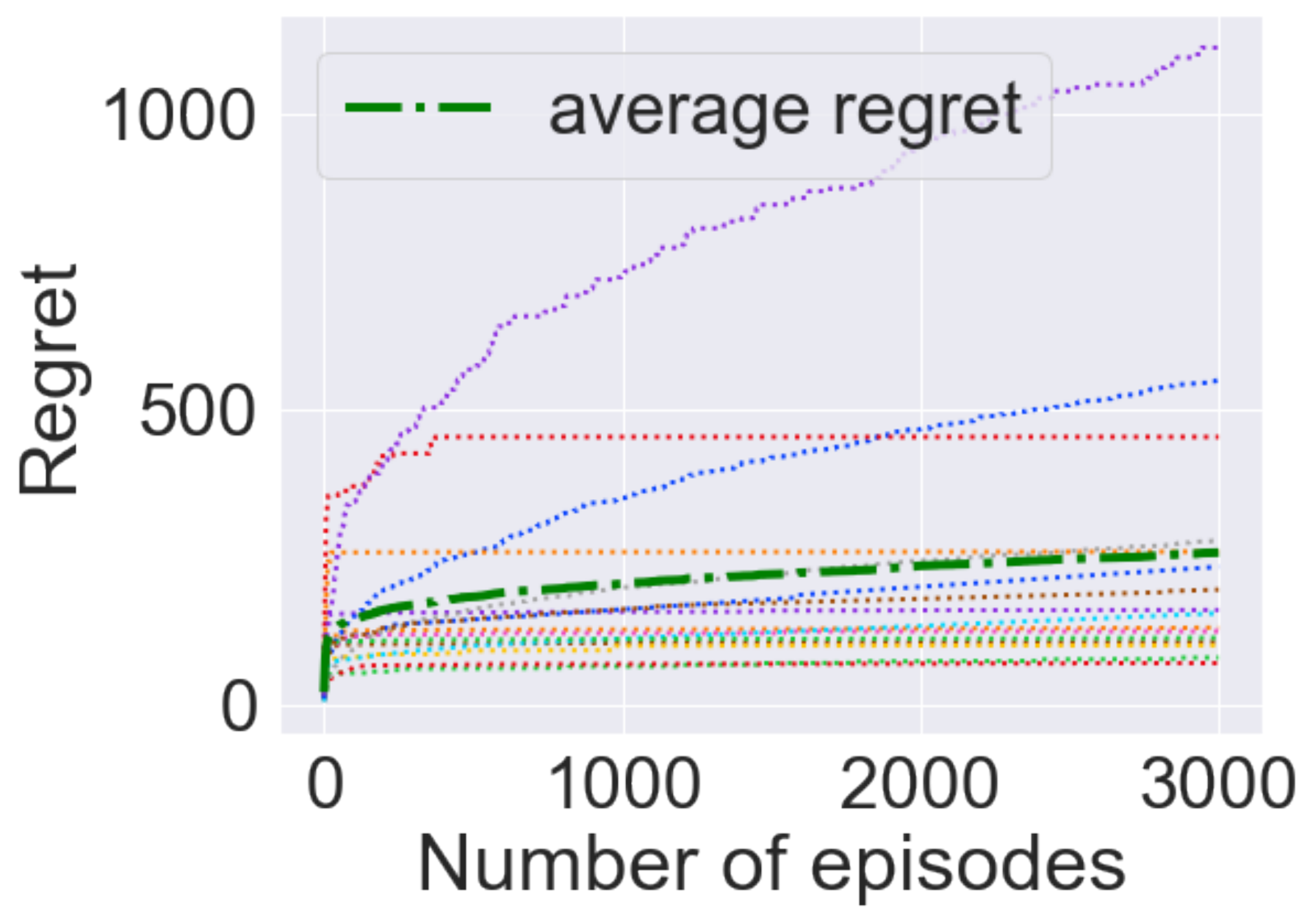}
        &\includegraphics[width=0.3\linewidth]{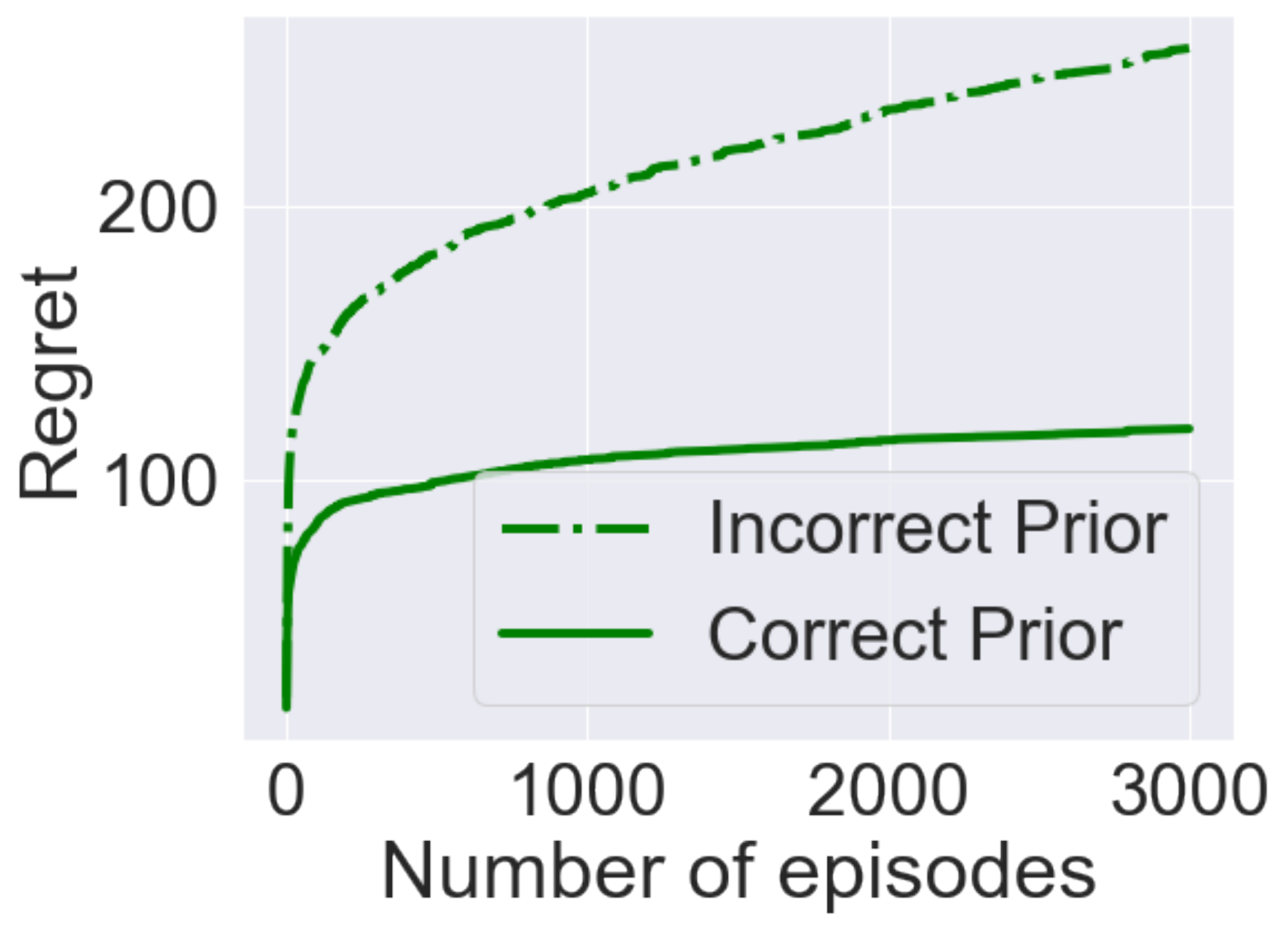}\\
        \rotatebox{90}{\qquad\qquad PSRL}
        &\subfigure[Correct Prior]{\label{subfig:gittinsPS_corr_prior}\includegraphics[width=0.3\linewidth]{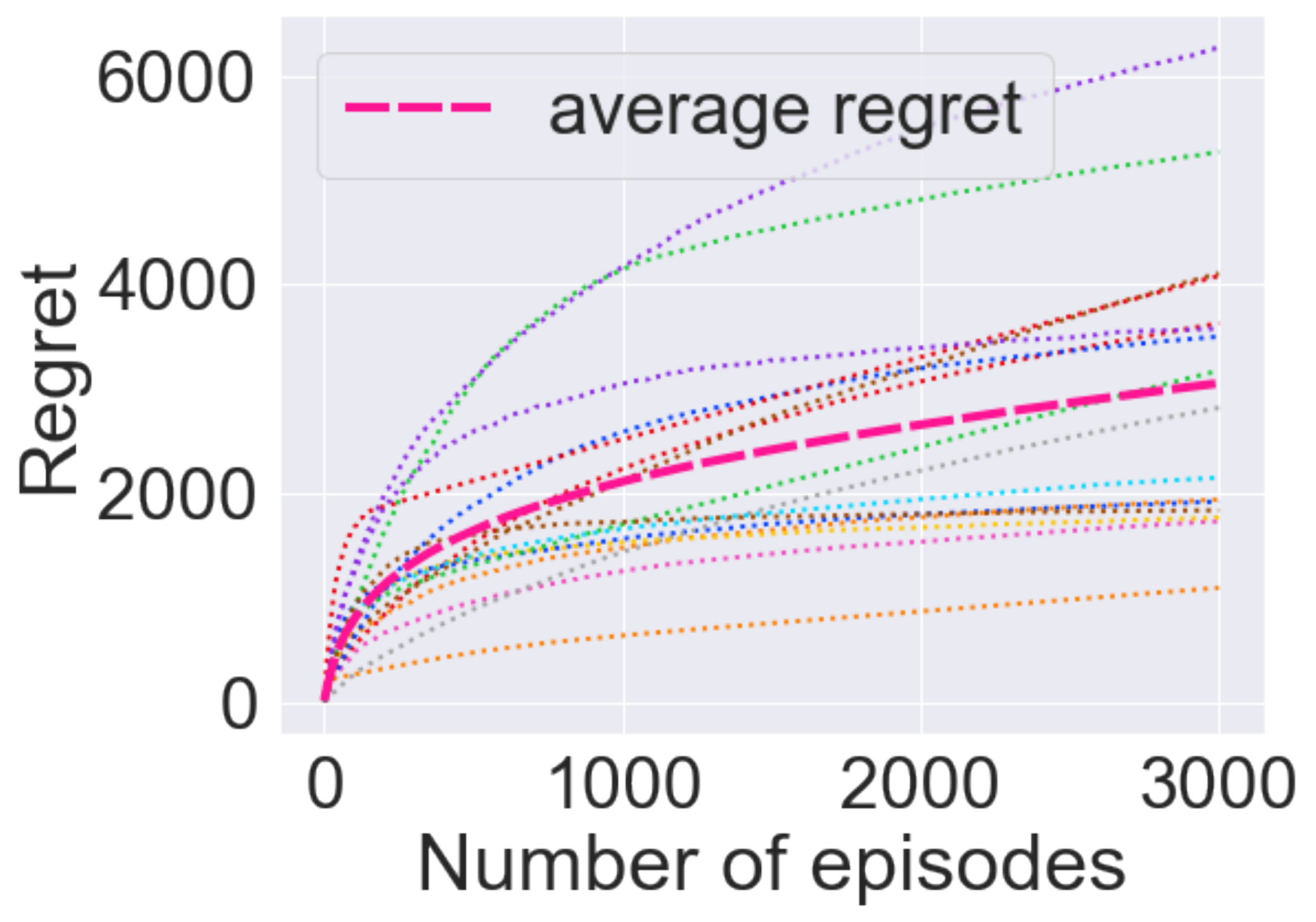}}
        &\subfigure[Incorrect Prior]{\label{subfig:gittinsPS_incorr_prior}\includegraphics[width=0.3\linewidth]{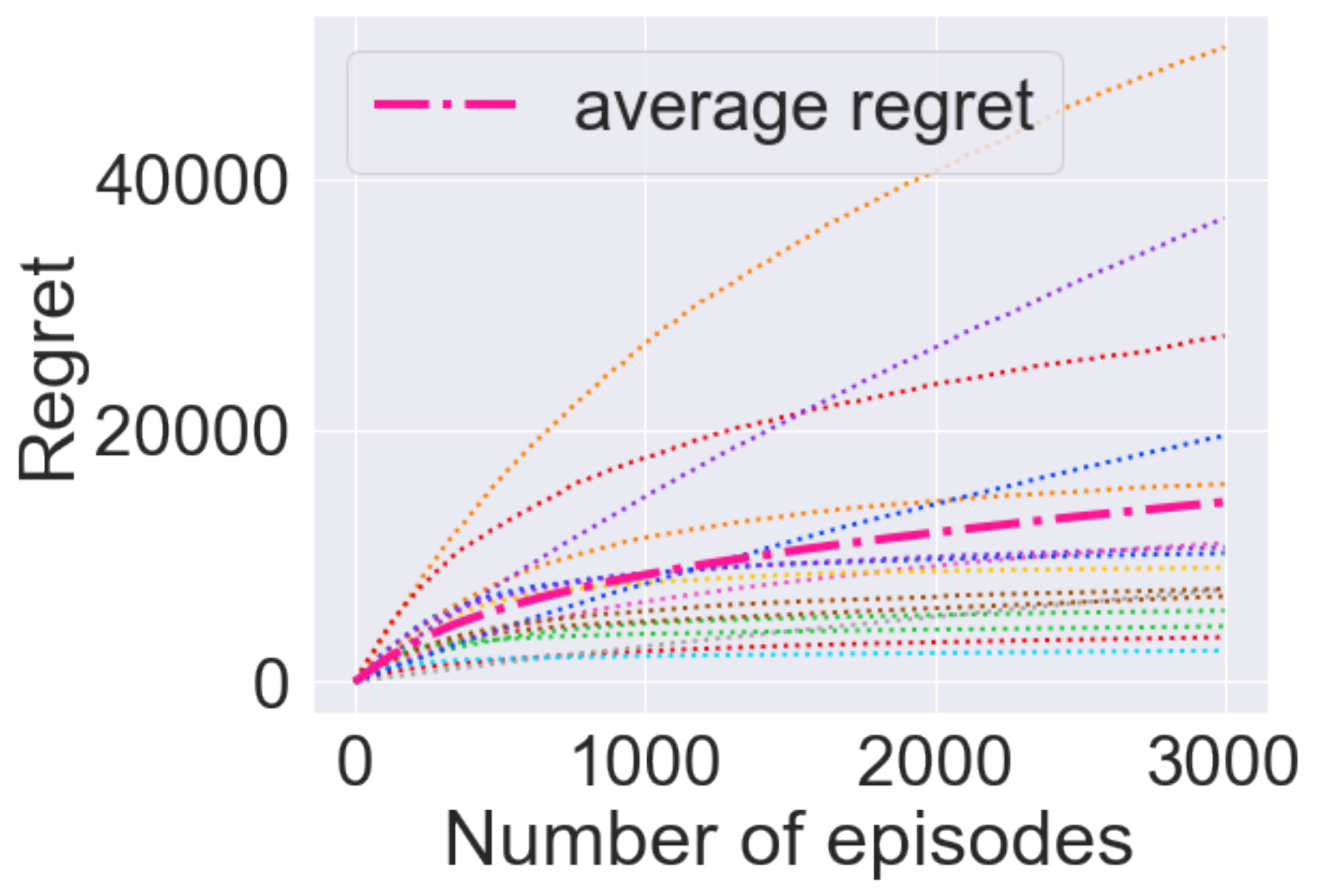}}
        &\subfigure[Bayesian Regret]{\label{subfig:gittinsPS_prior_choices}\includegraphics[width=0.3\linewidth]{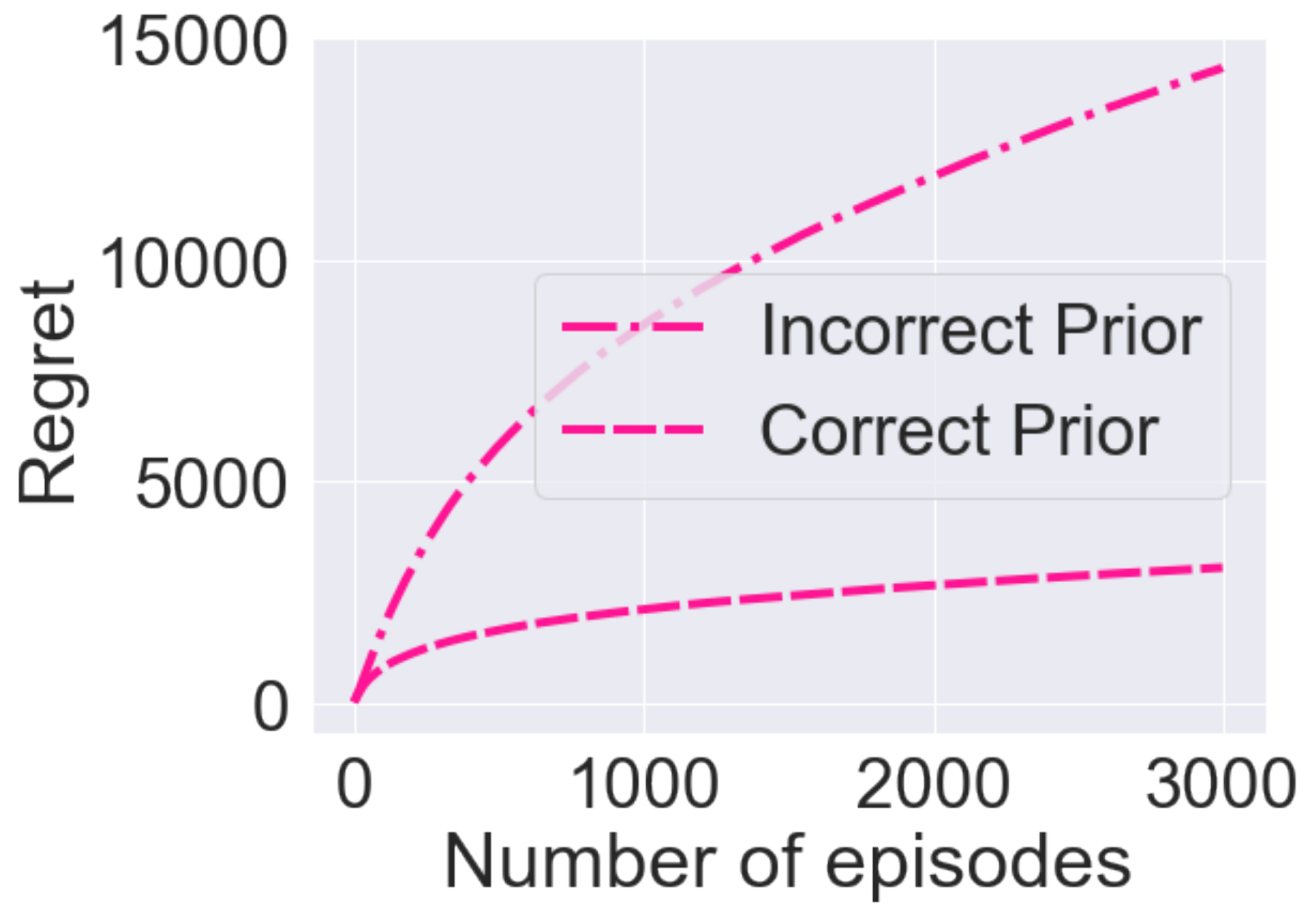}}
        \end{tabular}
        \caption{Bayesian regret of MB-PSRL and vanilla PSRL in 3 4-state Random Walk chains. For each chain, we draw 16 random models and run the algorithms for 5 simulations in each model (there are 80 simulations in total). In panels (a) and (b), we plot 16 dotted lines that correspond to the average cumulative regret over 5 simulations in the 16 samples. The solid and dash-dot lines are the average regret each over 80 simulations (the estimated Bayesian regret). \figurename~\ref{subfig:gittinsPS_corr_prior} shows the performance when reward prior is well chosen (namely, $U([0,1])$). \figurename~\ref{subfig:gittinsPS_incorr_prior} is when the reward prior is incorrectly chosen (namely Gaussian-Gamma distribution). \figurename~\ref{subfig:gittinsPS_prior_choices} compares the Bayesian regret of the correct prior with the incorrect one (dash-dot line). In both case, the prior of next state transition is well chosen (namely, Dirichlet distribution). Y-axis range changes for each figure.}
    \label{fig:bayes_gittinsPS}
\end{figure*}

\section{Experimental environment}
\label{apx:envi}

The code of all experiments is given in a separated zip file that contains all necessary material to reproduce the simulations and the figures. 

Our experiments were run on HPC platform with 1 node of 16 cores of Xeon E5. The experiments were made using Python 3 and Nix and submitted as supplementary material and will be made publicly available with the full release of the paper. The package requirement are detailed in README.md. Using only \emph{1 core} of Xeon E5, the Table~\ref{tab:sim_time} gives some orders of duration taken by each experiment (with discount factor $\beta=0.99$, and 3000 episodes per simulation). We would like to draw two remarks. First, the duration reported in Figure~\ref{fig:randomwalk_cpt_3b4s} is the time for policy computation (algorithm's parameters update and policy computation). The duration reported in Table~\ref{tab:sim_time} include this plus the computation time for oracle (because we track the regret), the state transition time along the trajectories of oracle and of each algorithm, resetting time... This explains why the duration reported in Table~\ref{tab:sim_time} cannot be compared to the duration reported in Figure~\ref{fig:randomwalk_cpt_3b4s}.  Second, the duration shown in Table~\ref{tab:sim_time} are meant to be a rough estimation of the computation time (we only ran the simulation once and the average duration might fluctuate). 

\begin{table}[ht]
\begin{tabular}{ |c|c|c|c|c|c| } 
\hline
Experiment & MB-PSRL  & PSRL & MB-UCRL2 & MB-UCBVI & Total\\
\hline
Scenario 1 & 40 min  & - & 3days & 50 min & 3days\\ 
\hline
Scenario 2 & 200 min  & - & - & 200 min & 400 min\\ 
\hline
Scenario 3 & 90 min  & 260 min & - & - & 350 min\\ 
\hline
\end{tabular}
\vspace{0.2cm}
\caption{Approximative execution time for simulating each algorithm and tracking its regret in each scenario. This time includes the time given in Figure~\ref{fig:randomwalk_cpt_3b4s} and the computation time needed by oracle (because we track the regret), the state transition time along the trajectories of oracle and each algorithm, etc. In each scenario, we set the discount factor $\beta=0.99$ and run the algorithms for $3000$ episodes per simulation.}
\label{tab:sim_time}
\end{table}

\end{document}